\newif\ifreport
\reporttrue

\ifreport
\documentclass[conference]{article}
\usepackage[margin=1in]{geometry}
\usepackage{authblk, hyperref}
\else
\documentclass[conference]{IEEEtran}
\fi
\usepackage{cite}
\usepackage{amsmath,amssymb,amsfonts, enumitem}
\usepackage{graphicx}
\usepackage{textcomp}
\usepackage{xcolor}
\usepackage[normalem]{ulem}
\usepackage{amsthm,algorithm,algorithmicx,algpseudocode,subfig}
\def\BibTeX{{\rm B\kern-.05em{\sc i\kern-.025em b}\kern-.08em
		T\kern-.1667em\lower.7ex\hbox{E}\kern-.125emX}}
\usepackage{tikz, pgfplots, tabularx, subfig}
\usepackage[thinc]{esdiff}
\usetikzlibrary{shapes, arrows, scopes, chains, pgfplots.groupplots,intersections,calc, patterns,arrows.meta}
\pgfplotsset{compat=1.16}
\usepgfplotslibrary{fillbetween}

\DeclareMathOperator*{\ucb}{UCB}
\DeclareMathOperator*{\lcb}{LCB}
\newtheorem{theorem}{Theorem}

\newtheorem{lemma}[theorem]{Lemma}

\newtheorem{remark}{Remark}
\newcommand{\prob}{\mathbb{P}}
\newcommand{\mut}{\tilde{\mu}}
\newcommand{\muh}{\hat{\mu}}
\newcommand\numberthis{\addtocounter{equation}{1}\tag{\theequation}}
\newcommand{\enc}{\mathsf{enc}}
\newcommand{\dec}{\mathsf{dec}}
\newcommand{\remove}[1]{}

\begin{document}

    \title{ICQ: A Quantization Scheme for Best-Arm Identification Over Bit-Constrained Channels}

\ifreport
\author[1]{Fathima Zarin Faizal}
\author[2]{Adway Girish}
\author[3]{Manjesh Kumar Hanawal}
\author[1]{Nikhil Karamchandani}
\affil[1]{Department of Electrical Engineering, IIT Bombay}
\affil[2]{School of Computer and Communication Sciences, EPFL}
\affil[3]{Industrial Engineering and Operations Research, IIT Bombay}

\fi
\date{April 25, 2023}

	\maketitle
	
	\begin{abstract}
    We study the problem of best-arm identification in a distributed variant of the multi-armed bandit setting, with a central learner and multiple agents. Each agent is associated with an arm of the bandit, generating stochastic rewards following an unknown distribution. Further, each agent can communicate the observed rewards with the learner over a bit-constrained channel. We propose a novel quantization scheme called Inflating Confidence for Quantization (ICQ) that can be applied to existing confidence-bound based learning algorithms such as Successive Elimination. We analyze the performance of ICQ applied to Successive Elimination and show that the overall algorithm, named ICQ-SE, has the order-optimal sample complexity as that of the (unquantized) SE algorithm. Moreover, it requires only an exponentially sparse frequency of communication between the learner and the agents, thus requiring considerably fewer bits than existing quantization schemes to successfully identify the best arm. We validate the performance improvement offered by ICQ with other quantization methods through numerical experiments.
	\end{abstract}
	\ifreport
    \tableofcontents \newpage
    \fi
	\tikzstyle{round} = [circle, draw, very thick, text centered, inner sep=1.5pt, text width=4em]
	\tikzstyle{block} = [rectangle, draw, very thick,
	text width=5em, text centered, rounded corners, minimum height=3em]
	\tikzstyle{line} = [draw, -latex']
	\section{Introduction}
The \emph{multi-armed bandit} (MAB) problem is a sequential decision-making model involving a learner and an environment, where in each round, the learner chooses from $K$ actions or arms. Each arm is associated with a reward distribution that is \textit{a priori} unknown to the learner. The learner then receives a random reward drawn from the distribution of the chosen arm. We consider the \emph{pure exploration} variant \cite{pac_2002} of this problem, where the learner is required to identify the best arm, i.e., the arm with the highest mean reward, with accuracy better than a prescribed confidence level. The learner is evaluated based on the number of samples (\emph{sample complexity}) it requires to identify the best arm. Thus a `desirable' algorithm is one that identifies the best arm with a smaller sample complexity, for a given confidence level. The pure exploration setting is well-studied \cite{audibert2010best, jamieson2014survey, bubeck2011pure} under the assumption that the learner gets to observe the reward values with full precision. 

We consider a distributed variant of the pure exploration MAB setup where the learner cannot observe the reward samples directly, but through intermediate agents that act as an interface for each arm. Unlike the traditional pure exploration MAB setup, the learner no longer has access to the rewards with full precision, i.e., the agents observe the rewards obtained and communicate aggregated information to the learner over noiseless, bit-constrained channels. A key point is that each agent `represents' a single arm, i.e., each agent pulls and observes rewards from only one fixed arm associated with it. 

\textbf{Motivation.} Such distributed learning setups with limited communication arise in many real-world systems. For example, in wireless networks with bandwidth constraints involving remote and low-complexity agents, the cost of communicating the rewards could become a performance bottleneck. Reducing the number of bits transmitted would result in lower power consumption and wireless interference. This is particularly significant in IoT networks where devices are typically resource-constrained and battery-powered \cite{dames2017detecting, savic2014belief, song2018recommender, ding2019interactive}. It may then be easier to have the agents process/compress the observations locally, before passing on information in a more condensed format. Finally, 
when learning from privacy-sensitive data, quantization can be used to hide the exact rewards obtained in such a way that the specific contents remain unintelligible to the learner, but there is still enough information to carry out the overall learning task, as in \cite{gandikota2020vqsgd}.

\textbf{Contributions.} To overcome constraints on the precision with which information can be sent from the agents to the learner, we develop a quantization scheme 
called \emph{Inflating Confidence for Quantization} (ICQ) for the best-arm identification problem. Our quantization scheme allows agents to communicate the reward information with fewer bits while still allowing the learner to extract enough information to identify the best arm. The key idea behind the scheme is to generate a high-probability range for the mean reward estimator at the agent that is smaller than the actual range of the rewards to reduce the range over which quantization needs to be done. This is done using appropriately defined confidence intervals for the quantized values.

While we build our scheme on top of the successive elimination framework proposed for the standard best-arm identification problem \cite{pac_2002}, and develop an algorithm called ICQ-SE, a key feature of our proposed quantization strategy is that it can be used in conjunction with a broad class of alternate schemes (such as LUCB \cite{kalyanakrishnan2012lucb} and lil'UCB \cite{jamieson2014lil}, to obtain corresponding algorithms ICQ-LUCB and ICQ-lil'UCB, and so on). This `universality' is clearly desirable and draws inspiration from \cite{hanna21} where the proposed quantization strategy had a similar feature with relation to a broad class of regret-minimization algorithms. Our algorithm ICQ-SE has the following features (shown in Section \ref{sec: analysis}):
\begin{enumerate}[label=\arabic*.]
    \item the learner needs to communicate with the agents only exponentially sparsely;
    \item requires only $B \geq 1$ number of bits for each round of communication for bounded rewards; 
    \item ensures order-optimal sample complexity compared to the distributed setup with no bit constraints; and
    \item can be easily modified to be used with other confidence bound based algorithms (see Remark \ref{rem: general-algo}).
\end{enumerate}
In Section \ref{sec: sims}, through simulations, we show that this scheme performs better than other quantization schemes in the MAB literature, for both bounded and unbounded rewards. 
	\section{Related work}
MAB problems are well explored in the literature in various settings like expected regret minimization, simple regret minimization, and Best Arm Identification (BAI) \cite{lattimore_bandit_2020}. For BAI problems \cite{audibert2010best, jamieson2014survey, bubeck2011pure, pac_2002, kalyanakrishnan2012lucb, jamieson2014lil, garivier2016track}, the goal is to identify the best arm either with high confidence within a given budget (\emph{fixed budget}) or with fewer samples for a given threshold on the probability of making a mistake (\emph{fixed confidence}). The algorithms developed for these settings assume that the learner has access to samples from the reward distributions with full precision. However, this assumption need not hold in federated setups \cite{AAA2021_FederatedMAB_ShiShen} where the learners and the agents need to exchange information, and any bottleneck in communication needs to be taken into account. Some recent works have addressed such issues by modeling communication bottlenecks as capacity constraints \cite{hanna21,mitra22, mitra_heterogeneity} or as a limited resource that comes at an additional cost \cite{srinivas_federated}.

Our work is closer to \cite{hanna21,mitra22, mitra_heterogeneity}, which propose quantization methods to improve the performance of learning algorithms under channel capacity constraints. 
In \cite{hanna21}, the authors propose a quantization scheme, named QuBan, that can be used over a large class of MAB algorithms in the regret minimization setting to achieve order-optimal regret. QuBan differs from our method as it does not make use of confidence bounds for the mean reward estimators. 
In \cite{mitra22}, the authors propose an adaptive quantization scheme and a decision-making policy for the Linear Stochastic Bandit setting and show that $B=\mathcal{O}(d)$ bits guarantees order-optimal regret, where $d$ is the dimension of the arm set. Their scheme is the closest to ours, where confidence bounds for the mean reward estimators are used to find a smaller range to quantize on, albeit for the regret minimization setting. Moreover, similar to us, they show that using 1 bit for each transmission from the agent to the learner is sufficient to ensure order-optimal regret compared to the unquantized setting when the reward distributions have bounded support.

In \cite{mitra_heterogeneity}, a pure exploration setting is considered, where a learner and multiple clients identify the best arm together, with each client being allotted a disjoint subset of the arms. Like us, they also propose a quantization scheme for communicating rewards between the clients and the learner; the difference being that their scheme only works with bounded rewards. Also, it is hard to control the number of bits used by their algorithm whereas our proposed scheme ICQ-SE provably allows for a trade-off between the amount of communication and  performance. See Section \ref{sec: sims} for a more detailed comparison.
	\section{Problem setup} \label{sec: problem_setup}
In this section, we define notation that will be used throughout the paper and formalize our system model.  The overall setup has been illustrated in Figure \ref{fig: setup}. Broadly, we study a distributed multi-armed bandit (MAB) problem where the decision-making and observing entities are separated. They must communicate their `results' (decisions or observations) to each other over a noiseless channel using a finite number of bits,  with the overall goal of performing a learning task. 
\begin{figure}[!htbp]
    \centering
		\begin{tikzpicture}[scale=0.75]
		\node [block, label={[align=center]above:Decode $\mathbf{s}_{1,i}, \mathbf{s}_{3,i}$\\ and broadcast $c_{i+1}$}]  at (0,3) (learner) {Learner};
		
		\node [round, label={[align=center]below: encode information\\into  $\mathbf{s}_{1,i}$}] (2)  at (-4,0) (agent1) {Agent 1} ;
		\path [line] (learner.west) -- node[sloped, anchor=center, above] {\footnotesize action $c_i$}(agent1.north);
		\path [line] (agent1) -- node[sloped, anchor=center, below, yshift=0pt, align=center] {\footnotesize finite-bit \\\footnotesize symbol $\mathbf{s}_{1,i}$}([xshift=0.3cm]learner.south west);
		
		\node [round, label={[align=center]below:Inactive in\\ comm. round $i$}] (2)  at (0,-1) (agent2) {Agent 2} ;
		\path [line] (learner.south) -- node[sloped, anchor=center, above, align=center] {\footnotesize action $c_i$}(agent2.north);

		\node [round, label={[align=center]below: encode information\\into  $\mathbf{s}_{3,i}$}] (2)  at (4,0) (agent3) {Agent 3} ;
		\path [line] (learner.east) -- node[sloped, anchor=center, above] {\footnotesize action $c_i$}(agent3.north);
		\path [line] (agent3) -- node[sloped, anchor=center, below, yshift=0pt, align=center] {\footnotesize finite-bit \\\footnotesize symbol $\mathbf{s}_{3,i}$}([xshift=-0.3cm]learner.south east);
		
	\end{tikzpicture}
	
	\caption{Block diagram illustrating the overall setup, shown here for the case with 3 agents, i.e., $K=3$.}
	\label{fig: setup}
\end{figure}
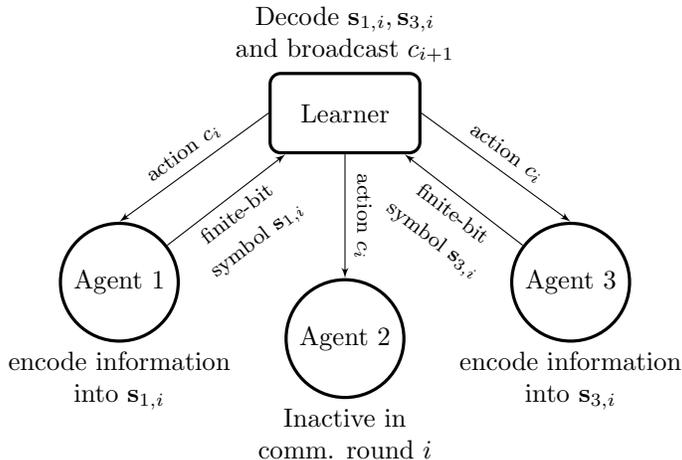 

\textbf{The distributed MAB.} There is a central learner, connected to each of the $K$ distributed agents, via a noiseless channel that is bit-constrained from the agent to the learner. The learner and the agents work together to solve a fixed confidence stochastic MAB problem consisting of $K$ arms. Agent $i \in \{1,\ldots,K\} \triangleq [K]$ has access to arm $i$ of the MAB, which is associated with a reward distribution $\nu_i$. We assume that the distributions $\{\nu_i\}_{i=1}^{K}$ are $\sigma^2$-subgaussian\footnote{A random variable $X$ is said to be $\sigma^2$-subgaussian if for any $t > 0$, $\mathbb{P} \left(|X - \mathbb{E}[X]|>t\right) \leq 2\exp \left(-t^2/2\sigma^2 \right)$} and bounded on $[a,b]$. For $i \in [K]$, let $\mu_i$ denote the mean of the distribution $\nu_i$ and $r_{i,t}$ denote the $t^{\text{th}}$ reward sample drawn from $\nu_i$. For each arm $i$, $\{r_{i,t}: t \ge 1\}$ is an i.i.d. process; furthermore the reward samples are independent across different arms. We also assume that arm 1 is the best arm for notational convenience, i.e., $\mu_1 \geq \mu_i$ where $i \neq 1$. Also, define the suboptimality gaps $\Delta_i = \mu_1 - \mu_i$ for $i \in [K]$. 

The broad objective for the learner is to identify the arm with the highest mean reward by sequentially selecting arms and sampling from their associated reward distributions. 

\textbf{The communication model.} Our communication model is summarized in Figure \ref{fig: setup}. Communication between the learner and the agents happens in rounds. At the beginning of \textit{(communication) round} $i$, based on the information the learner has seen till $i-1$, it broadcasts an action $c_i$ to all the $K$ agents, where $c_i$ encodes the information about the actions to be taken by each agent at round $i$. The agents respond to the learner after a fixed synchronized duration and the learner updates its estimate of the best arm based on the information it has received from the agents at round $i$. This constitutes one round. Each agent communicates at most once in each round.

We assume that each agent is only capable of collecting samples from its associated arm reward distribution, aggregating the information from the samples it has seen so far and transmitting it to the central learner. The agents cannot share information between each other and can communicate only with the central learner. This is commonly the case with low-complexity and resource-constrained devices such as drones and sensors. Moreover, they are only allowed to use a finite number of bits for each transmission to the learner. We assume that these transmissions are completed without any errors or erasures. Note that we do not assume that communication from the learner to the agents is bit-constrained and also do not put any computational restrictions on it; this is true for several application settings and  similar assumptions are also common in the literature \cite{mitra_heterogeneity,mitra22,hanna21}.  

\textbf{Performance metrics.} We consider the \emph{fixed confidence} BAI problem for the distributed MAB setting described above. The goal here is to find \emph{sound} strategies that find the optimal arm in a finite number of rounds for a given confidence level $\delta$. Formally, if the strategy stops after using $\tau_\delta$ samples and outputs arm $J_{\tau_\delta}$, a sound strategy ensures that $\mathbb{P}(\tau_\delta< \infty, J_{\tau_{\delta}} \neq 1) \leq \delta$. Since we are dealing with a communication-constrained setting, we also use the metrics of \emph{(communication) round complexity} $\tau_{r,\delta}$ (the number of communication rounds needed) and \emph{communication complexity} $ B_{\delta} $ (the total number of bits used by the algorithm) to study the performance of any sound strategy.

An online learning algorithm for this distributed MAB setup has the following components: (1) a \textit{sampling rule} at the learner that at each time considers the history of communicated messages received from the agents thus far, and prescribes the arm(s) to be pulled in the next round;  (2) a \textit{communication rule} at each agent that prescribes the rounds at which the agent will communicate with the learner and also specifies the content of the message; (3) a \textit{stopping rule} at the learner that specifies when the learner will stop sampling arms any further; and (4) a \textit{recommendation rule} at the learner that specifies its estimate for the best arm after the algorithm terminates.

\textbf{Objective.}
We develop learning algorithms for the distributed MAB setting outlined above where the agents need to limit the amount of communication with the learner. This is done via restricting both the frequency of communication (in terms of rounds) as well as the sizes of the messages exchanged (in terms of bits). Using too little communication can of course lead to a large penalty in terms of the sample complexity, and therein lies the main technical challenge of designing communication and quantization strategies which can effectively navigate this trade-off.

In the following sections, we propose a class of policies parameterized by the frequency of communication and the number of bits used in each message, and then explicitly characterize the impact of these parameters on the sample complexity. 

	\section{Proposed quantization scheme and algorithm} \label{sec: algo} 
This section is devoted to the description of our proposed quantization scheme \emph{Inflating Confidence for Quantization} (ICQ) and its application to Successive Elimination as ICQ-SE. Algorithm \ref{algo: learner_side} describes the actions to be taken by the learner while Algorithm \ref{algo: agent_side} describes the agent operation (definitions of the additional notation involved can be found in \eqref{eqn: u_prime}, \eqref{eqn: U_breakup} and \eqref{eqn: lcb_ucb}). We provide more details below. 
\begin{algorithm}[!htbp] 
	\caption{\textsc{ICQ-SE} algorithm (learner-side)}
	\label{algo: learner_side}
	\begin{algorithmic}[1]
		\Procedure{\textsc{ICQ-SE-learner}}{$K,\delta, B, \{b_i\}$}
		\State Let $S\leftarrow \{1,\ldots,K\}$
		\State For $1 \leq j \leq K$, let $\mut_{j,0}$ be sampled uniformly from $[a,b]$
		\For{$1 \leq i < \infty$}
		\For{$j \in S$}
		\State Instruct agent $j$ to sample $b_i$ times
		\State Receive quantized value $\mathbf{s}_{j,i}$ from agent $j$ 
		\State $L_{i,j} \leftarrow [\lcb(j,i-1,\delta)-U^\prime(i,\delta),\ \ucb(j,i-1,\delta)+U^\prime(i,\delta)]$
		\State Decode $\tilde{\mu}_{j,i} = \dec(\mathbf{s}_{j,i}, B, L_{i,j})$
		\EndFor
		\State $S \leftarrow S \setminus \{m \in S : \underset{j \in [K]}{\max}\lcb(j,i,\delta) \geq \ucb(m,i,\delta) \}$
		\State STOP if $|S|=1$
		\EndFor
		\State \textbf{return} only element in $S$
		\EndProcedure
	\end{algorithmic}
\end{algorithm}

\begin{algorithm}[!htbp] 
	\caption{\textsc{ICQ-SE} algorithm (agent-side)}
	\label{algo: agent_side}
	\begin{algorithmic}[1]
		\Procedure{\textsc{ICQ-SE-agent}$_j$}{$\delta, B,i,\mut_{j,i-1}$}
		\State Pull arm $j$ $b_i$ times
		\State $L_{i,j} \leftarrow [\lcb(j,i-1,\delta)-U^\prime(i,\delta),\ \ucb(j,i-1,\delta)+U^\prime(i,\delta)]$
		\State Send $\mathbf{s}_{j,i} = \enc(\hat{\mu}_{j,i}, B, L_{i,j})$
		\State \textbf{return} quantized value $\mathbf{s}_{j,i}$
		\EndProcedure
	\end{algorithmic}
\end{algorithm}

\textbf{Successive Elimination.} The broad idea behind the \emph{Successive Elimination} (SE) framework \cite{pac_2002} for a classical MAB setting (where there is only a learner observing full-precision rewards and no intermediate agents) is to characterize high-confidence bounds for the means of the distributions of each arm. One derives confidence widths $U'(i, \delta)$ such that the empirical mean $\muh_{j,i}$ of arm $j$ at round $i$ lies in the interval $[\hat{\mu}_{j,i} - U'(i,\delta),\hat{\mu}_{j,i} + U'(i,\delta)]$ around the actual mean $\mu_j$ with a `high' probability (we make this formal later). The upper limit is called the Upper Confidence Bound (UCB) and the lower limit is called the Lower Confidence Bound (LCB). The learner constantly keeps track of a set $S$ of \emph{active arms}, i.e., the set of arms still in contention to be the best arm. The set $S$ is initialized to be the set of all arms $[K]$. At the end of a round, if the UCB of any arm $k$ lies below the LCB of any other arm $j$, then arm $k$ is removed from the active set. Thus, under the high probability event that these confidence bounds contain the actual reward means, removing arms whose UCBs lie below the LCB of some other arm would guarantee that the algorithm is removing only suboptimal arms. The algorithm makes a mistake only when these high probability events do not occur. 

\textbf{High-level description of the algorithm.} In our setting, the learner no longer has access to full-precision rewards, and instead receives quantized estimates from the agents associated with each arm. In line with the SE framework, we also refer to the agents corresponding to the active arms as \emph{active agents}. As we would like to reduce the number of bits used, it is inefficient for the agent to communicate each sample that it sees. We thus consider a batched approach that ensures that communication happens in a sparse manner. The learner pulls active arms (through the agents) in batches; in particular, during (communication) round $i$, the agents pull and observe rewards from their associated arms $b_i$ times before sending (a summary of) the results to the learner. We also define $t_i$ to be the cumulative sum of arm pulls for each active arm till round $i$, i.e., $t_i = \sum_{j=1}^i b_j$. We show our results for an \emph{exponentially-sparse} communication framework similar to \cite{srinivas_federated}, i.e., $t_i = \alpha^i$ for some $\alpha > 1$.

At the beginning of round $i$, the learner instructs each agent in $S$, the set of active agents, to sample from their associated reward distributions $b_i$ times. At the end of round $i$, each active agent must have made a total of $t_i$ cumulative arm pulls and sends to the learner a quantized estimate of the empirical mean of the rewards obtained. The learner first decodes the quantized estimate, then decides which arms will remain in the active set using the SE framework. This update requires defining new confidence intervals that account for quantization; details will be provided later. This marks the end of a communication round. The algorithm terminates when there is only one arm left in the active set, which is the recommended arm. Before describing the working of the algorithm in more detail, it is instructive to look at the quantization part separately.

\textbf{Quantization scheme.}
Each agent calculates the empirical mean of the observed rewards, which must first be quantized and encoded into a bit string to be sent over the bit-constrained channel. Similarly, at the learner, we must be able to obtain a decoded estimate of the empirical mean from the encoded bit string. This is achieved as follows. We first fix an interval that we `expect' the empirical mean to belong to with high probability (this will become clear later), divide it into $2^B$ equal bins, then transmit a bit string that will be decoded at the learner as the midpoint of the bin.
We formalize this below.

Let $[\alpha,\beta]$ be the `expected' real interval as described above. First, divide $[\alpha,\beta]$ into $2^B$ bins of equal width $\frac{\beta-\alpha}{2^B}$, and associate with the midpoint of each bin, a $B$-length bit string. Then, the encoder $\enc(x, B, [\alpha,\beta])$ returns the bit string $\mathbf{s}$ associated with its nearest bin midpoint (even if $x \notin [\alpha, \beta]$), and the decoder $\dec(\mathbf{s}, B, [\alpha,\beta])$ returns the midpoint (a real number in $[\alpha, \beta]$) corresponding to the bit string $\mathbf{s}$. We may simply refer to them as $\enc(x)$ and $\dec(\mathbf{s})$ when $B$ and the interval are clear from context. The quantization scheme is summarized in Figure \ref{fig: quant_scheme}. Note that this quantization scheme uses $B$ bits for each transmission. Another point to note is that if $x \in [\alpha,\beta]$, the quantization error between $x$ and the decoded quantized value $\dec(\enc(x))$ is at most $\frac{\beta-\alpha}{2\cdot2^B}$, i.e., if $x$ is known to lie within the interval, then the quantization error is at most a factor of $\frac{1}{2^{B+1}}$ times the width of the interval.
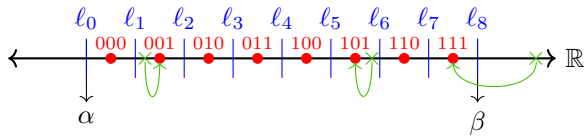
\begin{figure}[!htbp]
	\centering
	\begin{tikzpicture}[scale=1.3]
	\draw[thick,<->] (-2.8,0) node[anchor= north east, align=right] {} -- (2.8,0) node[anchor= west]{$\mathbb{R}$};
	\draw[->] (2,0.2) node[align=left, anchor=south]{} -- (2,-0.45) node[align=center, anchor=north]{ $\beta$};
	\draw[->] (-2,0.2) node[align=left, anchor=south]{} -- (-2,-0.45) node[align=center, anchor=north]{ $\alpha$};
	\foreach \i in {-4,...,4} \draw[-, color= blue] (\i/2,0.2) node[align=left, anchor=south]{ $\ell_{\the \numexpr \i+4 \relax}$} -- (\i/2,-0.2) node[align=center, anchor=north]{};
	\foreach \i in {-4,...,3} \filldraw[color=red] (\i/2+0.25,0) circle (0.05);
	
	\node[anchor=south, color=red] at (-1.75, 0) {\scriptsize 000};
	\node[anchor=south, color=red] at (-1.25, 0) {\scriptsize 001};
	\node[anchor=south, color=red] at (-0.75, 0) {\scriptsize 010};
	\node[anchor=south, color=red] at (-0.25, 0) {\scriptsize 011};
	\node[anchor=south, color=red] at (0.25, 0) {\scriptsize 100};
	\node[anchor=south, color=red] at (0.75, 0) {\scriptsize 101};
	\node[anchor=south, color=red] at (1.25, 0) {\scriptsize 110};
	\node[anchor=south, color=red] at (1.75, 0) {\scriptsize 111};
	
	\node[color = green!75!red, inner sep=0pt, outer sep=-2pt] (x1) at (2.6,0) {$\times$}; 
	\node[inner sep=0pt, outer sep=1pt] (mid8) at (1.75,0) {};
	\path[->, color=green!75!red] (x1) edge[bend right=-90, looseness=1] node [left] {} (mid8);
	
	\node[color = green!75!red, inner sep=0pt, outer sep=-2pt] (x2) at (-1.4,0) {$\times$};
	\node[inner sep=0pt, outer sep=1pt] (mid2) at (-1.25,0) {};
	\path[->, color=green!75!red] (x2) edge[out=-90, in=-90, looseness=8] node [left] {} (mid2);
	
	\node[color = green!75!red, inner sep=0pt, outer sep=-2pt] (x3) at (0.92,0) {$\times$};
	\node[inner sep=0pt, outer sep=1pt] (mid6) at (0.75,0) {};
	\path[->, color=green!75!red] (x3) edge[out=-90, in=-90, looseness=7] node [left] {} (mid6);
	
\end{tikzpicture}
	\caption{Illustration of the quantization scheme when $B = 3$ --- the blue lines given by $\ell_i$ mark the separation between the $2^B$ equal bins, the red points denote the midpoints of these bins, and the green `$\times$'s (the values to be quantized) get mapped to their nearest midpoints.}
	\label{fig: quant_scheme}
\end{figure} 

\textbf{Details of the algorithm.} At round $i$, consider the agent $j \in [K]$ making its $k$\textsuperscript{th} cumulative arm pull, where $1 \leq k \leq t_i$. It observes a reward $r_{j,k}$. At the end of this round, it calculates the empirical mean of the rewards from arm $j$ observed over all rounds upto and including round $i$, $\muh_{j,i} = \frac{1}{t_i} \sum_{k=1}^{t_i} r_{j,k}$. By defining the confidence width
\begin{equation}
	U^\prime(i,\delta) = \sigma \sqrt{\frac{2 \log( 4Kt_i^2/\delta)}{t_i}}, \label{eqn: u_prime}
\end{equation}
for each round $i$ and arbitrary $\delta > 0$, it \ifreport can be shown (as in the proof of Lemma \ref{lem: delta}) \fi from the subgaussian concentration inequality \cite{lattimore_bandit_2020} and a union bound that
\begin{equation}
	\mathbb{P}\left(\cup_{i \geq 1}\cup_{j \in [K]}  |\hat{\mu}_{j,i} - \mu_j | > U^\prime(i,\delta)\right) \leq \delta. \label{eqn: p_less_delta}
\end{equation}
Thus, at any round $i$, the actual mean of arm $j$ lies in the interval $[\hat{\mu}_{j,i} - U^\prime(i,\delta),\hat{\mu}_{j,i} + U^\prime(i,\delta)]$ w.h.p. However, since the communication channel is bit-constrained, the agents cannot simply transmit the infinite precision real number $\muh_{j,i}$ as is --- they instead transmit a quantized version of $\muh_{j,i}$ as described above. Let $\mut_{j,i}$ be the decoded estimate of the mean of arm $j$ that the learner recovers at the end of round $i$, i.e., $\mut_{j,i} = \dec(\enc(\muh_{j,i}))$.

To account for the potential increase in error due to the quantization necessitated by the bit-constrained channel, we introduce a slack in the confidence interval through a different confidence width $U(i, \delta)$. The goal is to obtain a concentration bound for $\mut_{j,i} - \mu_j$ in terms of $U(i, \delta)$, in a form similar to \eqref{eqn: p_less_delta}. We now provide an intuitive explanation to motivate an expression of $U(i, \delta)$ that achieves exactly this. (Note that this is not meant to be a proof; that this does indeed work will be shown in Section \ref{sec: analysis}). First note that a straightforward application of the triangle inequality gives
\begin{equation}
	|\tilde{\mu}_{j,i} - \mu_j| \leq |\tilde{\mu}_{j,i}  - \hat{\mu}_{j,i}| + |\hat{\mu}_{j,i} - \mu_j|. \label{eqn: triangle}
\end{equation}
The first term corresponds to the quantization error and the second term corresponds to the error in the empirical mean itself. An interval in which the latter lies w.h.p.\ is taken care of by the bound \eqref{eqn: p_less_delta}, so it is enough to establish such an interval for the quantization error. Recall from the `Quantization scheme' description before Figure \ref{fig: quant_scheme} that the quantization error is at most $\frac{1}{2^{B+1}}$ times the width of the interval if the empirical mean $\muh_{j,i}$ is known to originally lie in the interval. Thus, our task is to find an appropriate interval in which $\muh_{j,i}$ lies w.h.p.\ to perform the quantization.

As the latest estimate of the mean that the learner has access to at the beginning of round $i$ is the quantized estimate at round $i-1$, i.e., $\mut_{j,i-1}$, we construct the interval to be centered around $\mut_{j,i-1}$. We also want that this interval contain the new empirical mean at round $i$, i.e., $\hat{\mu}_{j,i}$, w.h.p. We now find a recursive expression for $U(i, \delta)$ that inductively satisfies these properties. We first make the inductive assumption that $\mut_{j, i-1}$ lies in $[\mu_j - U(i-1, \delta), \mu_j+ U(i-1, \delta)]$ w.h.p. Combined with the knowledge that $\muh_{j,i}$ lies in $[\mu_j - U^\prime(i, \delta), \mu_j+ U^\prime(i, \delta)]$ w.h.p. from \eqref{eqn: p_less_delta}, we have that the interval $[\tilde{\mu}_{j,i-1} - U^\prime(i,\delta) - U(i-1,\delta),\tilde{\mu}_{j,i-1} + U^\prime(i,\delta) + U(i-1, \delta)]$ contains $\muh_{j,i}$ w.h.p. This is illustrated in Figure \ref{fig: U_Uprime}. Note that we have found an interval that we `expect' the empirical mean to belong to, as promised when defining the quantization scheme relative to an interval. Thus, we perform the quantization over an interval of width $2[U^\prime(i,\delta) +  U(i-1, \delta)]$ centered at $\mut_{j,i-1}$, and hence, the quantization error is at most $\frac{1}{2^B}[U^\prime(i,\delta) +  U(i-1, \delta)]$.

\begin{figure}[!htbp]
	\centering
	\begin{tikzpicture}[scale=1]
	\draw[thick,<->] (-4,0) node[anchor= north east, align=right] {} -- (4,0) node[anchor= west]{$\mathbb{R}$};
	\draw[<-] (0,1.5) node[align=left, anchor=south]{\small $\mut_{j,i-1}$} -- (0,-0.18) node[align=center, anchor=north]{};
	\draw[<-] (2,0.3) node[align=left, anchor=south]{\small $\ucb(j,i-1,\delta)$} -- (2,-0.18) node[align=center, anchor=north]{};
	\draw[<-] (-2,0.3) node[align=left, anchor=south]{\small $\lcb(j,i-1,\delta)$} -- (-2,-0.18) node[align=center, anchor=north]{};
	\draw[<->, color=blue] (0,1) node[anchor= south west, align=left] {\hspace{0.7em}\small $U(i-1, \delta)$} -- (2,1) node[anchor= north west]{};
	\draw[<->, color=blue] (0,1) node[anchor= north west]{} -- (-2,1) node[anchor= south west, align=left] {\hspace{0.7em}\small $U(i-1, \delta)$};
	
	\draw[->] (1.3,0.18) node[align=left, anchor=south]{} -- (1.3,-0.6) node[align=center, anchor=north]{\small $\mu_j$};
	\draw[<->, color=red] (1.3,-0.22) node[anchor= north west, align=left] {\hspace{0.5em}\small $U^\prime(i, \delta)$} -- (2.8,-0.22) node[anchor= north west]{};
	\draw[<->, color=red] (1.3,-0.22) node[anchor= north west]{} -- (-0.2,-0.22) node[anchor= north west, align=left] {\hspace{0.5em}\small $U^\prime(i, \delta)$};
	\draw[-, dashed, black!25] (2.8,0) node[align=left, anchor=south]{} -- (2.8,-1) node[align=center, anchor=north]{};
	\draw[-, dashed, black!25] (-0.2,0) node[align=left, anchor=south]{} -- (-0.2,-1) node[align=center, anchor=north]{};
	
	\draw [red,decorate,decoration={brace,amplitude=5pt,mirror},xshift=0pt,yshift=0pt]
	(-0.2,-1) -- (2.8,-1) node [black,midway,xshift=0cm,yshift=-0.7cm, text width=3.5cm, align=center] 
	{\footnotesize $\muh_{j,i}$ guaranteed to be here given this position of $\mu_j$};
	
	\draw[<->, color=red] (2,1) node[anchor= north west]{} -- (3.5,1) node[anchor= south east, align=left] {\hspace{0.3em}\small $U^\prime(i, \delta)$};
	\draw[<->, color=red] (-3.5,1) node[anchor= north west]{} -- (-2,1) node[anchor= south east, align=left] {\hspace{0.6em}\small $U^\prime(i, \delta)$};
	
	\draw[-, dashed, black!25] (2,0.3) node[align=left, anchor=south]{} -- (2,1.7) node[align=center, anchor=north]{};
	\draw[-, dashed, black!25] (-2,0.3) node[align=left, anchor=south]{} -- (-2,1.7) node[align=center, anchor=north]{};
	
	\draw[-, dashed, black!25] (3.5,0) node[align=left, anchor=south]{} -- (3.5,2) node[align=center, anchor=north]{};
	\draw[-, dashed, black!25] (-3.5,0) node[align=left, anchor=south]{} -- (-3.5,2) node[align=center, anchor=north]{};
	\draw [red!50!blue,decorate,decoration={brace,amplitude=5pt},xshift=0pt,yshift=0pt]
	(-3.5,2) -- (3.5,2) node [black,midway,xshift=0cm,yshift=0.4cm] 
	{\footnotesize $\muh_{j,i}$ guaranteed to be here in general, relative to $\mut_{j, i-1}$};
\end{tikzpicture}
	\caption{Motivation for defining $U(i, \delta)$ as \eqref{eqn: U_breakup} -- Start with $\mut_{j, i-1}$, then by the inductive hypothesis, $\mu_j$ should lie within $U(i-1, \delta)$ of $\mut_{j, i-1}$ (blue interval). From \eqref{eqn: p_less_delta}, $\muh_{j,i}$ lies within $U^\prime(i, \delta)$ of $\mu_j$ (red interval). Putting them together, $\muh_{j,i}$ lies within $U(i-1,\delta) + U^\prime(i, \delta)$ of $\mut_{j, i-1}$. }
	\label{fig: U_Uprime}
\end{figure}
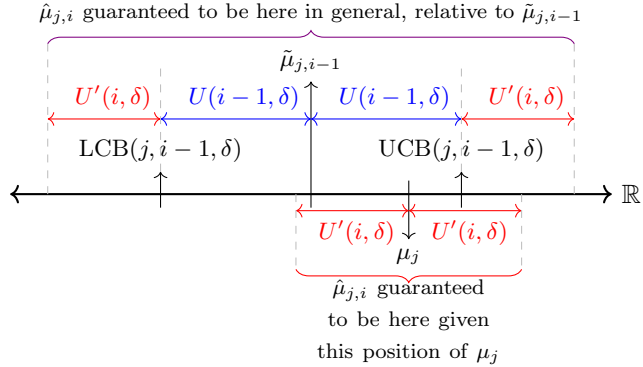

Motivated by the above and \eqref{eqn: triangle}, define, for $i \geq 1, j \in [K]$,
\begin{align*}
	U(i,\delta) 	&= \frac{1}{2^{B}}\left[U^\prime(i,\delta) +  U(i-1, \delta)\right] + U^\prime(i, \delta) \numberthis \label{eqn: U_breakup},
\end{align*}
where we set $U(0, \delta) =  b-a$ to ensure that the induction hypothesis (namely, that $\mut_{j,i-1} $ satisfies its confidence bounds) holds for the index $i-1 = 0$. Now that we have an appropriate confidence width $U(i, \delta)$ that we expect (heuristically, from the above discussion; this will be proved in Section \ref{sec: analysis}) to provide a similar confidence bound for $\mut_{j,i}$ as \eqref{eqn: p_less_delta} does for $\muh_{j,i}$, we define the appropriate LCB and UCB for our algorithm, as 
\begin{equation} \label{eqn: lcb_ucb}
\begin{split}
	\lcb(j,i,\delta) &= \mut_{j,i} - U(i, \delta),\\
	\ucb(j,i,\delta) &= \mut_{j,i} + U(i, \delta).
\end{split}    
\end{equation}

	\section{Analysis of the ICQ-SE algorithm} \label{sec: analysis}
The following result can be shown for the sample complexity of SE run on the distributed MAB setup outlined above if the channel from the agent to the learner is not bit-constrained:
\begin{theorem} \label{thm: se_results}
SE is a sound algorithm. Moreover, with probability at least $1-\delta$, it successfully identifies the best arm using at most
	\begin{align*}
		\mathcal{O} \left( \sum_{j \neq 1} \frac{102 \alpha \sigma^2}{\Delta_j^2} \ln \left(\frac{64 \sigma^2 \sqrt{4K\delta}}{\Delta_j^2}\right) + 1 \right)
	\end{align*} 
 samples. 
\end{theorem}
We now show that ICQ-SE is a sound algorithm and also analyze its sample complexity. We restrict our attention to ICQ-SE just to provide concrete results, but a similar analysis can be carried out for other confidence bound-based algorithms as well (see Remark \ref{rem: general-algo}). We do so by a sequence of lemmas and theorems, similar to a standard analysis of Successive Elimination-type algorithms, such as in \cite{pac_2002} and \cite{srinivas_federated}. The main novelties in our work are Lemma \ref{lem: equivalence} and Theorem \ref{thm: sample_complexity}, where we relate the confidence intervals $U(i, \delta)$ and $U^\prime(i, \delta)$ at the agents and the learner respectively, thereby allowing us to prove similar results as for vanilla Successive Elimination. \ifreport Proofs of all lemmas and theorems stated below can be found in the Appendix.\fi

Recall that $\mu_j$ is the mean of arm $j \in [K]$, $\muh_{j,i}$ is the empirical mean of arm $j$ at round $i$ at the agent (which will then be encoded and sent to the learner), and $\mut_{j,i}$ is the decoded estimate of the mean of arm $j$ at round $i$ at the learner. Note that the only concentration bound we know about these quantities \textit{a priori} is \eqref{eqn: p_less_delta}, which relates $\muh_{j,i}$ and $U^\prime(i, \delta)$. The learner, however, observes $\mut_{j,i}$ and constructs confidence intervals of width $U(i,\delta)$. Our goal is for the learner to be able to identify w.h.p. the best arm in $[K]$. To this end, it is desirable to have a concentration bound on the quantities available at the learner, namely, $\mut_{j,i}$ and $U(i, \delta)$. 

Lemma \ref{lem: equivalence} below relates the following  --- (1) the event that in some round $i$ (hence the union over rounds), the estimated mean of arm $j$, $\muh_{j,i}$,   falls outside the confidence interval of width $U^\prime(i, \delta)$ centered about the true mean $\mu_j$, and (2) the identical event at the learner, except with the decoded mean $\mut_{j,i}$ and confidence width $U(i, \delta)$.

\begin{lemma} \label{lem: equivalence}
	For all arms $1 \leq j \leq K$, and any $\delta > 0$,
	\begin{align*}
		\bigcup_{i = 1}^\infty \left\{ |\mut_{j,i} - \mu_j | > U(i,\delta) \right\} \subseteq \bigcup_{i=1}^\infty \left\{ |\muh_{j,i} - \mu_j | > U^\prime(i,\delta) \right\}.
	\end{align*}
\end{lemma}
Using Lemma \ref{lem: equivalence}, we obtain the desired confidence bound on the decoded means at the learner, stated formally below. 
\begin{lemma} \label{lem: delta}
	For any $\delta > 0$, define the event
	\begin{equation*}
		\mathcal{E} := \bigcup_{j= 1}^K\bigcup_{i=1}^\infty  \left\{ |\mut_{j,i} - \mu_j | > U(i,\delta) \right\},
	\end{equation*}
	then $
	\prob \left( \mathcal{E} \right ) \leq \delta $.
\end{lemma}
The above lemma establishes that the event $\mathcal{E}^c$, wherein for each arm $j$ and at every round $i$, the decoded estimate at the learner $\muh_{j,i}$ is sufficiently close to the actual mean $\mu_j$, occurs with large probability. It then follows that with high probability, any arm that is eliminated during the successive elimination procedure must be suboptimal\ifreport, as stated in Theorem \ref{thm: best_arm}\fi. We now show that ICQ-SE is a sound algorithm and provide an upper bound for its sample complexity in the exponentially sparse regime (i.e., with $t_i = \alpha^i$) in Theorem \ref{thm: sample_complexity}.

\ifreport
\begin{theorem}\label{thm: best_arm}
	With probability $\geq 1 - \delta$, the best arm remains in the active set S until termination.
\end{theorem}

To prove our main result in Theorem \ref{thm: sample_complexity}, we require a technical lemma that relates the confidence widths $U(i, \delta)$ and $U^\prime(i, \delta)$.

\begin{lemma} \label{lem: U_ub}
	Consider $B \geq 1$ and let $t_i=\alpha^i$ where $\alpha \in \mathbb{N}$ such that $\alpha < 2^{2B}$. Then for $i \geq 1$,
	\begin{align*}
		U(i,\delta) \leq 2c\, U^\prime(i, \delta),
	\end{align*}
	where 
	\begin{align*}
		c = \left(1 + \frac{2}{2^{B}}\right)\frac{2^{B}}{2^{B} - \sqrt{\alpha}}.
	\end{align*}
\end{lemma}
\fi

\begin{theorem} \label{thm: sample_complexity}
	Consider $B \geq 1$ and let $t_i=\alpha^i$ where $\alpha \in \mathbb{N}$ and $1<\alpha < 2^{2B}$. With probability at least $1-\delta$, ICQ-SE will terminate and successfully identify the best arm after using 
	\begin{align*}
		\mathcal{O} \left( \sum_{j \neq 1} \frac{410 \alpha c^2 \sigma^2}{\Delta_j^2} \ln \left(\frac{256c^2 \sigma^2 \sqrt{4K\delta}}{\Delta_j^2}\right) + 1 \right)
	\end{align*} 
	samples, 
 	\begin{align*}
		\mathcal{O} \left( \sum_{j \neq 1} \log_{\alpha} \left ( \frac{410 \alpha c^2 \sigma^2}{\Delta_j^2} \ln \left(\frac{256c^2 \sigma^2 \sqrt{4K\delta}}{\Delta_j^2}\right)  + 1 \right) \right )
	\end{align*} 
 rounds, and
  	\begin{align*}
		\mathcal{O}  \left( B \sum_{j \neq 1} \log_{\alpha} \left ( \frac{410 \alpha c^2 \sigma^2}{\Delta_j^2} \ln \left(\frac{256c^2 \sigma^2 \sqrt{4K\delta}}{\Delta_j^2}\right)  + 1 \right) \right )
	\end{align*}
 bits, where \ifreport $c$ is as defined in Lemma \ref{lem: U_ub}. \fi
\end{theorem}
Comparing Theorem \ref{thm: sample_complexity} with the equivalent result for Successive Elimination with no quantization in Theorem \ref{thm: se_results}, we see that the upper bound for the sample complexity is worse only by a constant factor, i.e., we have order-optimal performance.

Additionally, note that $c$ depends on the choice of $B$ and $\alpha$, and in fact decreases as $B$ and $\alpha$ increase. Combined with the upper bound on sample complexity in Theorem~\ref{thm: sample_complexity}, we thus see a trade-off between the performance of the algorithm and the number of bits that it is allowed to use. In addition to the above theoretical result, we also investigate this trade-off via numerical simulations in Section \ref{sec: sims}.

\begin{remark} \label{rem: general-algo}
	The quantization scheme ICQ can in fact be used for any other algorithm that uses confidence bounds, such as LUCB \cite{kalyanakrishnan2012lucb} and lil'UCB \cite{jamieson2014lil}, to obtain algorithms ICQ-LUCB and ICQ-lil'UCB. This is because the crux of this scheme is the recursive definition for $U(\cdot,\cdot)$ that we obtain by separating the quantization error and the error in the empirical mean itself, via \eqref{eqn: triangle}. A similar  analysis can be carried out for ICQ-LUCB and ICQ-lil'UCB too, that we omit here for brevity.
\end{remark}

\begin{remark} \label{rem: unbounded}
    We use the boundedness of the rewards only in the definition of $U(0,\delta)$. Recall that the algorithm proceeds with an initial random guess for the mean of each arm $\{\tilde{\mu}_{j,0}\}_{j=1}^K$. As $U(i,\delta)$ is defined in a recursive fashion, for the induction in the proof of Lemma \ref{lem: equivalence} to hold, $U(0,\delta)$ needs to be such that $\{\tilde{\mu}_{j,0}\}_{j=1}^K$ are good guesses for the actual means with high probability satisfying Lemma \ref{lem: equivalence}. For $[a, b]$-bounded rewards, this holds trivially by taking $U(0, \delta) = (b - a)$, as $\left\{ |\mut_{j,0} - \mu_j | >  U(0, \delta) \right\} =  \left\{ |\mut_{j,0} - \mu_j | >  (b-a) \right\} = \emptyset$.
    For unbounded rewards, however, we can no longer use a constant number of bits in each round. Specifically, in round 1, it is not possible to obtain a high-probability guess for the mean by using a bounded number of bits. Nonetheless, we can use a different quantization scheme just for the first round to ensure that the quantization error is bounded. 
    The problem is then reduced to that with bounded rewards from round 2, whence we can use ICQ-SE as is. In Section \ref{sec: sims}, we demonstrate this on Gaussian rewards by running QuBan \cite{hanna21} (designed for unbounded rewards) in the first round. 
\end{remark}
	\section{Numerical experiments} \label{sec: sims}
In this section, we present results of numerical experiments comparing the performance of ICQ-SE with other quantization algorithms proposed for multi-armed bandits in the literature. In addition to the unquantized setting as a baseline, we compare ICQ-SE with the quantization schemes  QuBan \cite{hanna21} and Fed-SEL \cite{mitra_heterogeneity}. Each of these schemes is implemented on top of the same batched Successive Elimination algorithm that ICQ-SE uses to highlight the difference between the quantization schemes. The algorithms are compared based on their (expected) sample complexity $E\left [ \tau_{\delta} \right ]$, (expected) round complexity $E\left [ \tau_{r,\delta} \right ]$ and (expected) communication complexity $E\left [ B_{\delta} \right ]$. In all our experiments, the performance of each algorithm was averaged over 4000 iterations.

QuBan \cite{hanna21} was proposed for the regret minimization setting where each sample that the agent observes is quantized and sent to the learner. A key feature of QuBan is that the agent uses shorter codewords to quantize samples close to the current estimate of the mean at the learner, while reward samples which are farther away are assigned longer codewords. This helps to minimize the expected number of bits used at each round. While this is a sound approach, it could result in a higher number of bits being used unnecessarily for our framework. QuBan has a parameter $\epsilon > 0$ that provides a trade-off between the number of bits used and the performance of the algorithm (a smaller value of $\epsilon$ provides a smaller regret using a higher number of bits).

In Fed-SEL \cite{mitra_heterogeneity}, in each round $i$, the entire interval $[a,b]$ is divided into bins of length $U'(i,\delta)$ and the empirical mean is quantized to the midpoint of one of these bins. The main drawback of this approach is that the number of bits used at each round is inversely proportional to the confidence bound at each round, i.e., the number of bits used per round grows with the number of rounds, making it hard to control the cumulative number of bits that the algorithm uses. 

\ifreport
The first set of experiments in Figures \ref{fig: QSE_alpha_round}, \ref{fig: QSE_alpha_sample}, \ref{fig: QSE_alpha_bits}, \ref{fig: QSE_B_round} and \ref{fig: QSE_B_sample} analyze the dependence of the performance of ICQ-SE on the parameters $\alpha$ (controlling the sparsity of communication) and $B$ (the number of bits in each transmission). We consider a five-armed multi-armed bandit instance where each arm is associated with a $\mathrm{Beta}(\gamma,1-\gamma)$ distribution, with $\gamma$ generated uniformly at random from $[0,1]$. In Figures \ref{fig: QSE_alpha_round} and \ref{fig: QSE_alpha_sample}, we observe that the number of communication rounds used by ICQ-SE to converge decreases with $\alpha$ while the number of samples used increases with $\alpha$. This is expected because increasing $\alpha$ results in sparser communication between the agents and the learner reducing the round complexity while the number of samples used increases. Figure \ref{fig: QSE_alpha_bits} shows the dependence of the cumulative number of bits used by the algorithm to converge (which we call \textit{communication complexity}) on $\alpha$. The decrease in the communication complexity with $\alpha$ is a natural consequence of the decrease in the communication round complexity.
\fi

Figures \ref{fig: bounded_round}, \ref{fig: bounded_sample}, and \ref{fig: bounded_bits} compare the performance of ICQ-SE, QuBan and Fed-SEL. We consider a five-armed multi-armed bandit instance where each arm is associated with a bounded support reward distribution, in particular the $\mathrm{Beta}(\gamma,1-\gamma)$ distribution with $\gamma$ generated uniformly at random from $[0,1]$. We observe that ICQ-SE with $B=3$ performs comparably with QuBan ($\epsilon=0.5$), and better than QuBan ($\epsilon=2$) and Fed-SEL in terms of sample and round complexity, while using a much lesser number of bits than all of them.

In Figures \ref{fig: Gaussian_round}, \ref{fig: Gaussian_sample}, and \ref{fig: Gaussian_bits}, we compare the performance of ICQ-SE and QuBan when the reward distributions associated with the arms are Gaussian (and hence unbounded; recall Remark \ref{rem: unbounded}). We consider a five-armed multi-armed bandit instance where each arm is associated with a Gaussian reward distribution of standard deviation 0.125 whose means are generated uniformly at random from the interval $[0,N]$, where $N$ is a sample drawn from a Gaussian distribution $\mathcal{N}(0,9)$. We use QuBan with $\epsilon=2$ for the first round of ICQ-SE as discussed in Remark \ref{rem: unbounded}. We again observe that ICQ-SE with $B=3$ performs comparably with the others in terms of sample and round complexity while using a much lesser number of bits. We make no comparison with Fed-SEL in this case because it is unclear how to extend the scheme to unbounded rewards, since it starts by dividing the finite-size reward range into bins of length $U^\prime(i, \delta)$.

The final set of experiments in Figures \ref{fig: Hardness_round}, \ref{fig: Hardness_sample} and \ref{fig: Hardness_bits} compare the dependence of the performance of ICQ-SE and QuBan on the hardness of the underlying bandit instance when the  reward distributions associated with the arms are Gaussian. We consider a five-armed multi-armed bandit instance where each arm is associated with a Gaussian distribution of standard deviation 0.125. Four of the arms have mean 0 and the remaining arm has mean $\Delta \in [0,1]$. A lower $\Delta$ implies that the mean of the optimal arm is closer to that of the non-optimal arms, resulting in a harder instance. As expected, the performance of all the algorithms improves as the hardness of the instance decreases. Moreover, we see the same trend with ICQ-SE ($B=3$) as earlier, especially on the harder instances.

\ifreport 
Finally, in Figure \ref{fig: sims_ext}, we also numerically analyze the impact of varying $\alpha$ (controlling the sparsity of communication) and $B$ (the number of bits in each transmission) on the performance of ICQ-SE. 
\fi

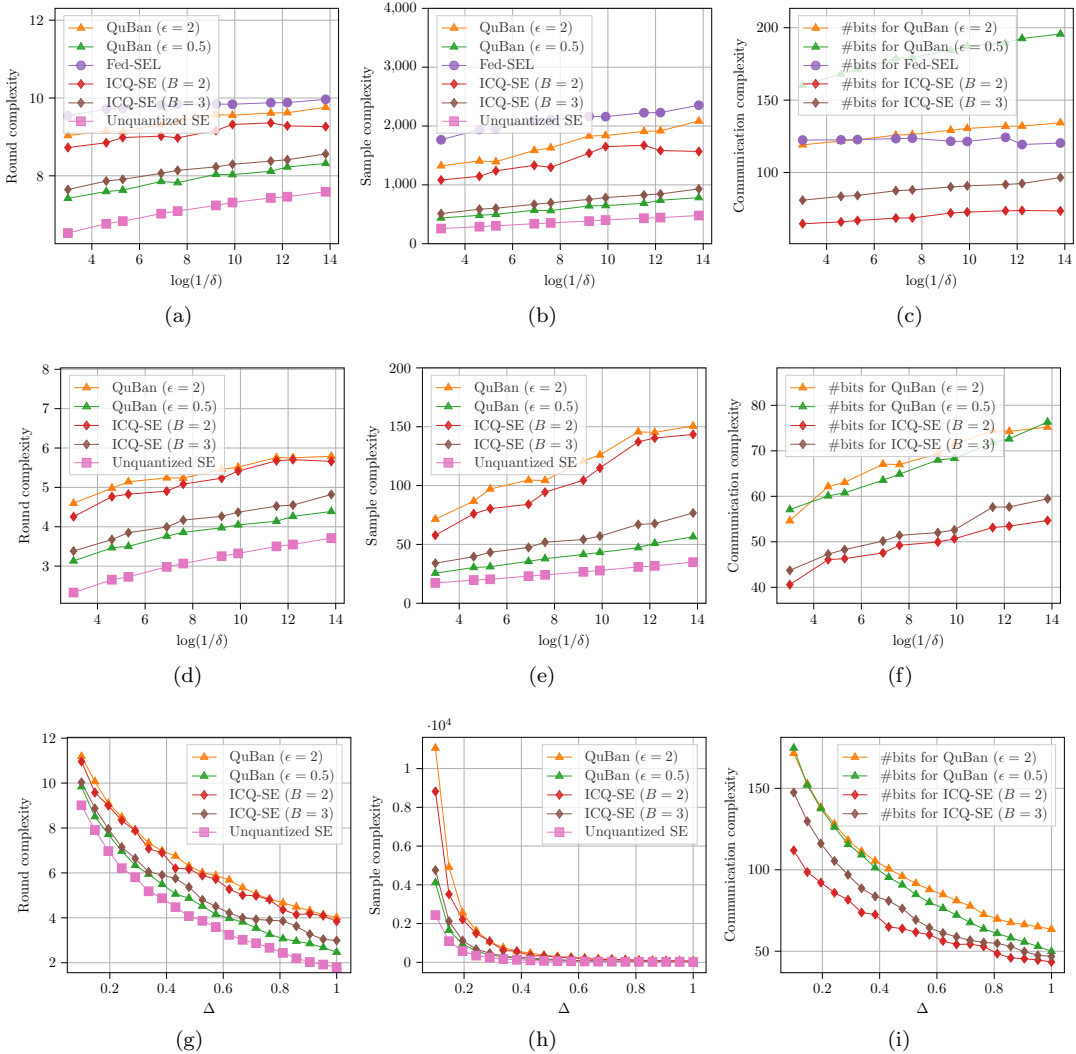
\begin{figure*}[!htbp]
	\centering
\subfloat[\label{fig: bounded_round}
]{\scalebox{0.55}{
\begin{tikzpicture}

\definecolor{crimson2143940}{RGB}{214,39,40}
\definecolor{darkgray176}{RGB}{176,176,176}
\definecolor{darkorange25512714}{RGB}{255,127,14}
\definecolor{forestgreen4416044}{RGB}{44,160,44}
\definecolor{lightgray204}{RGB}{204,204,204}
\definecolor{mediumpurple148103189}{RGB}{148,103,189}
\definecolor{orchid227119194}{RGB}{227,119,194}
\definecolor{sienna1408675}{RGB}{140,86,75}
\definecolor{steelblue31119180}{RGB}{31,119,180}

\begin{axis}[
legend cell align={left},
legend style={fill opacity=0.8, draw opacity=1, text opacity=1,   at={(0.03,0.97)},
  anchor=north west, draw=lightgray204},
tick align=outside,
tick pos=left,
x grid style={darkgray176},
xlabel={\(\displaystyle \log(1/\delta)\)},
xmajorgrids,
xmin=2.45474335933348, xmax=14.3564994721848,
xtick style={color=black},
y grid style={darkgray176},
ylabel={Round complexity},
ymajorgrids,
ymin=6.255, ymax=12.305,
ytick style={color=black}
]
\addplot [semithick, darkorange25512714, mark=triangle*, mark size=3, mark options={solid}]
table {%
13.8155105579643 9.75975
12.2060726455302 9.62225
11.5129254649702 9.6145
9.90348755253613 9.55825
9.21034037197618 9.57225
7.60090245954208 9.39075
6.90775527898214 9.32
5.29831736654804 9.1325
4.60517018598809 9.15625
2.99573227355399 9.03525
};
\addlegendentry{QuBan ($\epsilon=2$)}
\addplot [semithick, forestgreen4416044, mark=triangle*, mark size=3, mark options={solid}]
table {%
13.8155105579643 8.31825
12.2060726455302 8.22725
11.5129254649702 8.118
9.90348755253613 8.03
9.21034037197618 8.03825
7.60090245954208 7.82425
6.90775527898214 7.85525
5.29831736654804 7.63075
4.60517018598809 7.5965
2.99573227355399 7.42075
};
\addlegendentry{QuBan ($\epsilon=0.5$)}
\addplot [semithick, mediumpurple148103189, mark=*, mark size=3, mark options={solid}]
table {%
13.8155105579643 9.96575
12.2060726455302 9.88025
11.5129254649702 9.87875
9.90348755253613 9.84025
9.21034037197618 9.844
7.60090245954208 9.838
6.90775527898214 9.8155
5.29831736654804 9.72525
4.60517018598809 9.72
2.99573227355399 9.5415
};
\addlegendentry{Fed-SEL}
\addplot [semithick, crimson2143940, mark=diamond*, mark size=3, mark options={solid}]
table {%
13.8155105579643 9.2645
12.2060726455302 9.284
11.5129254649702 9.361
9.90348755253613 9.32
9.21034037197618 9.16575
7.60090245954208 8.976
6.90775527898214 9.02075
5.29831736654804 8.98325
4.60517018598809 8.85
2.99573227355399 8.72675
};
\addlegendentry{ICQ-SE ($B=2$)}
\addplot [semithick, sienna1408675, mark=diamond*, mark size=3, mark options={solid}]
table {%
13.8155105579643 8.5625
12.2060726455302 8.41325
11.5129254649702 8.37775
9.90348755253613 8.29675
9.21034037197618 8.22975
7.60090245954208 8.14075
6.90775527898214 8.06475
5.29831736654804 7.9095
4.60517018598809 7.86675
2.99573227355399 7.64575
};
\addlegendentry{ICQ-SE ($B=3$)}
\addplot [semithick, orchid227119194, mark=square*, mark size=3, mark options={solid}]
table {%
13.8155105579643 7.58975
12.2060726455302 7.4635
11.5129254649702 7.431
9.90348755253613 7.3145
9.21034037197618 7.2395
7.60090245954208 7.0935
6.90775527898214 7.028
5.29831736654804 6.83575
4.60517018598809 6.76775
2.99573227355399 6.53
};
\addlegendentry{Unquantized SE}
\end{axis}

\end{tikzpicture}} }
\subfloat[\label{fig: bounded_sample}
]{\scalebox{0.55}{
\begin{tikzpicture}

\definecolor{crimson2143940}{RGB}{214,39,40}
\definecolor{darkgray176}{RGB}{176,176,176}
\definecolor{darkorange25512714}{RGB}{255,127,14}
\definecolor{forestgreen4416044}{RGB}{44,160,44}
\definecolor{lightgray204}{RGB}{204,204,204}
\definecolor{mediumpurple148103189}{RGB}{148,103,189}
\definecolor{orchid227119194}{RGB}{227,119,194}
\definecolor{sienna1408675}{RGB}{140,86,75}
\definecolor{steelblue31119180}{RGB}{31,119,180}

\begin{axis}[
legend cell align={left},
legend style={
  fill opacity=0.8,
  draw opacity=1,
  text opacity=1,
  at={(0.03,0.97)},
  anchor=north west,
  draw=lightgray204
},
tick align=outside,
tick pos=left,
x grid style={darkgray176},
xlabel={\(\displaystyle \log(1/\delta)\)},
xmajorgrids,
xmin=2.45474335933348, xmax=14.3564994721848,
xtick style={color=black},
y grid style={darkgray176},
ylabel={Sample complexity},
ymajorgrids,
ymin=0, ymax=4000,
ytick style={color=black}
]
\addplot [semithick, darkorange25512714, mark=triangle*, mark size=3, mark options={solid}]
table {%
13.8155105579643 2084.842
12.2060726455302 1918.098
11.5129254649702 1908.736
9.90348755253613 1838.064
9.21034037197618 1830.6
7.60090245954208 1629.33
6.90775527898214 1585.106
5.29831736654804 1392.763
4.60517018598809 1406.082
2.99573227355399 1323.395
};
\addlegendentry{QuBan ($\epsilon=2$)}
\addplot [semithick, forestgreen4416044, mark=triangle*, mark size=3, mark options={solid}]
table {%
13.8155105579643 783.388
12.2060726455302 738.821
11.5129254649702 686.525
9.90348755253613 648.305
9.21034037197618 642.994
7.60090245954208 562.339
6.90775527898214 568.721
5.29831736654804 498.9815
4.60517018598809 482.659
2.99573227355399 437.4165
};
\addlegendentry{QuBan ($\epsilon=0.5$)}
\addplot [semithick, mediumpurple148103189, mark=*, mark size=3, mark options={solid}]
table {%
13.8155105579643 2353.136
12.2060726455302 2228.288
11.5129254649702 2226.156
9.90348755253613 2159.604
9.21034037197618 2163.816
7.60090245954208 2112.116
6.90775527898214 2088.108
5.29831736654804 1948.092
4.60517018598809 1941.088
2.99573227355399 1766.436
};
\addlegendentry{Fed-SEL}
\addplot [semithick, crimson2143940, mark=diamond*, mark size=3, mark options={solid}]
table {%
13.8155105579643 1568.174
12.2060726455302 1583.18
11.5129254649702 1672.52
9.90348755253613 1650.118
9.21034037197618 1538.44
7.60090245954208 1298.132
6.90775527898214 1331.644
5.29831736654804 1239.552
4.60517018598809 1146.056
2.99573227355399 1082.466
};
\addlegendentry{ICQ-SE ($B=2$)}
\addplot [semithick, sienna1408675, mark=diamond*, mark size=3, mark options={solid}]
table {%
13.8155105579643 931.158
12.2060726455302 845.387
11.5129254649702 827.714
9.90348755253613 783.5
9.21034037197618 752.619
7.60090245954208 696.138
6.90775527898214 669.769
5.29831736654804 602.854
4.60517018598809 586.576
2.99573227355399 510.795
};
\addlegendentry{ICQ-SE ($B=3$)}
\addplot [semithick, orchid227119194, mark=square*, mark size=3, mark options={solid}]
table {%
13.8155105579643 479.659
12.2060726455302 443.797
11.5129254649702 433.398
9.90348755253613 402.524
9.21034037197618 384.673
7.60090245954208 353.491
6.90775527898214 340.097
5.29831736654804 304.006
4.60517018598809 289.719
2.99573227355399 257.9615
};
\addlegendentry{Unquantized SE}
\end{axis}

\end{tikzpicture}} }
\subfloat[\label{fig: bounded_bits}
]{\scalebox{0.55}{
\begin{tikzpicture}

\definecolor{crimson2143940}{RGB}{214,39,40}
\definecolor{darkgray176}{RGB}{176,176,176}
\definecolor{darkorange25512714}{RGB}{255,127,14}
\definecolor{forestgreen4416044}{RGB}{44,160,44}
\definecolor{lightgray204}{RGB}{204,204,204}
\definecolor{mediumpurple148103189}{RGB}{148,103,189}
\definecolor{sienna1408675}{RGB}{140,86,75}
\definecolor{steelblue31119180}{RGB}{31,119,180}

\begin{axis}[
legend cell align={left},
legend style={
  fill opacity=0.8,
  draw opacity=1,
  text opacity=1,
  at={(0.03,0.97)},
  anchor=north west,
  draw=lightgray204
},
tick align=outside,
tick pos=left,
x grid style={darkgray176},
xlabel={\(\displaystyle \log(1/\delta)\)},
xmajorgrids,
xmin=2.45474335933348, xmax=14.3564994721848,
xtick style={color=black},
y grid style={darkgray176},
ylabel={Communication complexity},
ymajorgrids,
ymin=50.8394, ymax=213.3736,
ytick style={color=black}
]
\addplot [semithick, darkorange25512714, mark=triangle*, mark size=3, mark options={solid}]
table {%
13.8155105579643 134.36475
12.2060726455302 132.0435
11.5129254649702 131.96075
9.90348755253613 130.45975
9.21034037197618 129.21975
7.60090245954208 126.294
6.90775527898214 125.79325
5.29831736654804 122.64975
4.60517018598809 121.739
2.99573227355399 119.11325
};
\addlegendentry{\#bits for QuBan ($\epsilon=2$)}
\addplot [semithick, forestgreen4416044, mark=triangle*, mark size=3, mark options={solid}]
table {%
13.8155105579643 195.6675
12.2060726455302 192.53025
11.5129254649702 189.53775
9.90348755253613 187.21475
9.21034037197618 184.38825
7.60090245954208 179.38375
6.90775527898214 178.81025
5.29831736654804 171.191
4.60517018598809 168.0235
2.99573227355399 160.17425
};
\addlegendentry{\#bits for QuBan ($\epsilon=0.5$)}
\addplot [semithick, mediumpurple148103189, mark=*, mark size=3, mark options={solid}]
table {%
13.8155105579643 120.37925
12.2060726455302 119.34625
11.5129254649702 124.2875
9.90348755253613 121.473
9.21034037197618 121.58575
7.60090245954208 123.687
6.90775527898214 123.4145
5.29831736654804 122.71725
4.60517018598809 122.56325
2.99573227355399 122.407
};
\addlegendentry{\#bits for Fed-SEL}
\addplot [semithick, crimson2143940, mark=diamond*, mark size=3, mark options={solid}]
table {%
13.8155105579643 73.496
12.2060726455302 73.7455
11.5129254649702 73.5835
9.90348755253613 72.6885
9.21034037197618 72.037
7.60090245954208 68.631
6.90775527898214 68.525
5.29831736654804 66.862
4.60517018598809 65.897
2.99573227355399 64.601
};
\addlegendentry{\#bits for ICQ-SE ($B=2$)}
\addplot [semithick, sienna1408675, mark=diamond*, mark size=3, mark options={solid}]
table {%
13.8155105579643 96.555
12.2060726455302 92.39625
11.5129254649702 91.74
9.90348755253613 90.7935
9.21034037197618 90.03
7.60090245954208 87.99075
6.90775527898214 87.4455
5.29831736654804 84.14775
4.60517018598809 83.54625
2.99573227355399 80.87925
};
\addlegendentry{\#bits for ICQ-SE ($B=3$)}
\end{axis}

\end{tikzpicture}} }

\subfloat[\label{fig: Gaussian_round}
]{\scalebox{0.55}{
\begin{tikzpicture}

\definecolor{crimson2143940}{RGB}{214,39,40}
\definecolor{darkgray176}{RGB}{176,176,176}
\definecolor{darkorange25512714}{RGB}{255,127,14}
\definecolor{forestgreen4416044}{RGB}{44,160,44}
\definecolor{lightgray204}{RGB}{204,204,204}
\definecolor{mediumpurple148103189}{RGB}{148,103,189}
\definecolor{sienna1408675}{RGB}{140,86,75}
\definecolor{steelblue31119180}{RGB}{31,119,180}
\definecolor{orchid227119194}{RGB}{227,119,194}
\begin{axis}[
legend cell align={left},
legend style={
  fill opacity=0.8,
  draw opacity=1,
  text opacity=1,
  at={(0.03,0.97)},
  anchor=north west,
  draw=lightgray204
},
tick align=outside,
tick pos=left,
x grid style={darkgray176},
xlabel={\(\displaystyle \log(1/\delta)\)},
xmajorgrids,
xmin=2.45474335933348, xmax=14.3564994721848,
xtick style={color=black},
y grid style={darkgray176},
ylabel={Round complexity},
ymajorgrids,
ymin=2.0567375, ymax=8.0355125,
ytick style={color=black}
]
\addplot [semithick, darkorange25512714, mark=triangle*, mark size=3, mark options={solid}]
table {%
13.8155105579643 5.791
12.2060726455302 5.75125
11.5129254649702 5.758
9.90348755253613 5.51325
9.21034037197618 5.45775
7.60090245954208 5.23675
6.90775527898214 5.242
5.29831736654804 5.143
4.60517018598809 4.98175
2.99573227355399 4.596
};
\addlegendentry{QuBan ($\epsilon=2$)}
\addplot [semithick, forestgreen4416044, mark=triangle*, mark size=3, mark options={solid}]
table {%
13.8155105579643 4.38825
12.2060726455302 4.2645
11.5129254649702 4.1405
9.90348755253613 4.047
9.21034037197618 3.97375
7.60090245954208 3.85725
6.90775527898214 3.76275
5.29831736654804 3.5005
4.60517018598809 3.4665
2.99573227355399 3.13125
};
\addlegendentry{QuBan ($\epsilon=0.5$)}
\addplot [semithick, crimson2143940, mark=diamond*, mark size=3, mark options={solid}]
table {%
13.8155105579643 5.6605
12.2060726455302 5.70475
11.5129254649702 5.6755
9.90348755253613 5.41375
9.21034037197618 5.233
7.60090245954208 5.08275
6.90775527898214 4.90475
5.29831736654804 4.828
4.60517018598809 4.764
2.99573227355399 4.25525
};
\addlegendentry{ICQ-SE ($B=2$)}
\addplot [semithick, sienna1408675, mark=diamond*, mark size=3, mark options={solid}]
table {%
13.8155105579643 4.82025
12.2060726455302 4.5505
11.5129254649702 4.525
9.90348755253613 4.36875
9.21034037197618 4.2655
7.60090245954208 4.1695
6.90775527898214 3.99325
5.29831736654804 3.84925
4.60517018598809 3.67475
2.99573227355399 3.38325
};
\addlegendentry{ICQ-SE ($B=3$)}
\addplot [semithick, orchid227119194, mark=square*, mark size=3, mark options={solid}]
table {%
13.8155105579643 3.713
12.2060726455302 3.55025
11.5129254649702 3.502
9.90348755253613 3.3235
9.21034037197618 3.25175
7.60090245954208 3.066
6.90775527898214 2.98075
5.29831736654804 2.72775
4.60517018598809 2.655
2.99573227355399 2.3285
};
\addlegendentry{Unquantized SE}
\end{axis}

\end{tikzpicture}} }
\subfloat[\label{fig: Gaussian_sample}
]{\scalebox{0.55}{
\begin{tikzpicture}

\definecolor{crimson2143940}{RGB}{214,39,40}
\definecolor{darkgray176}{RGB}{176,176,176}
\definecolor{darkorange25512714}{RGB}{255,127,14}
\definecolor{forestgreen4416044}{RGB}{44,160,44}
\definecolor{lightgray204}{RGB}{204,204,204}
\definecolor{mediumpurple148103189}{RGB}{148,103,189}
\definecolor{sienna1408675}{RGB}{140,86,75}
\definecolor{steelblue31119180}{RGB}{31,119,180}
\definecolor{orchid227119194}{RGB}{227,119,194}
\begin{axis}[
legend cell align={left},
legend style={
  fill opacity=0.8,
  draw opacity=1,
  text opacity=1,
  at={(0.03,0.97)},
  anchor=north west,
  draw=lightgray204
},
tick align=outside,
tick pos=left,
x grid style={darkgray176},
xlabel={\(\displaystyle \log(1/\delta)\)},
xmajorgrids,
xmin=2.45474335933348, xmax=14.3564994721848,
xtick style={color=black},
y grid style={darkgray176},
ylabel={Sample complexity},
ymajorgrids,
ymin=0, ymax=200,
ytick style={color=black}
]
\addplot [semithick, darkorange25512714, mark=triangle*, mark size=3, mark options={solid}]
table {%
13.8155105579643 150.565
12.2060726455302 145.2515
11.5129254649702 145.618
9.90348755253613 126.0445
9.21034037197618 120.86
7.60090245954208 104.297
6.90775527898214 104.659
5.29831736654804 97.097
4.60517018598809 86.762
2.99573227355399 71.4625
};
\addlegendentry{QuBan ($\epsilon=2$)}
\addplot [semithick, forestgreen4416044, mark=triangle*, mark size=3, mark options={solid}]
table {%
13.8155105579643 56.5205
12.2060726455302 50.9045
11.5129254649702 47.234
9.90348755253613 43.198
9.21034037197618 41.4765
7.60090245954208 37.8425
6.90775527898214 35.627
5.29831736654804 31.046
4.60517018598809 30.3745
2.99573227355399 25.584
};
\addlegendentry{QuBan ($\epsilon=0.5$)}
\addplot [semithick, crimson2143940, mark=diamond*, mark size=3, mark options={solid}]
table {%
13.8155105579643 143.464
12.2060726455302 140.384
11.5129254649702 137.189
9.90348755253613 114.9105
9.21034037197618 104.517
7.60090245954208 94.405
6.90775527898214 84.092
5.29831736654804 80.3715
4.60517018598809 76.012
2.99573227355399 57.8125
};
\addlegendentry{ICQ-SE ($B=2$)}
\addplot [semithick, sienna1408675, mark=diamond*, mark size=3, mark options={solid}]
table {%
13.8155105579643 76.694
12.2060726455302 67.685
11.5129254649702 66.952
9.90348755253613 57.033
9.21034037197618 54.178
7.60090245954208 51.8925
6.90775527898214 47.4045
5.29831736654804 43.282
4.60517018598809 39.5625
2.99573227355399 34.0285
};
\addlegendentry{ICQ-SE ($B=3$)}
\addplot [semithick, orchid227119194, mark=square*, mark size=3, mark options={solid}]
table {%
13.8155105579643 34.903
12.2060726455302 31.7625
11.5129254649702 30.8815
9.90348755253613 27.842
9.21034037197618 26.688
7.60090245954208 24.1115
6.90775527898214 23.056
5.29831736654804 20.3675
4.60517018598809 19.7415
2.99573227355399 17.2445
};
\addlegendentry{Unquantized SE}
\end{axis}

\end{tikzpicture}} }
\subfloat[\label{fig: Gaussian_bits}
]{\scalebox{0.55}{
\begin{tikzpicture}

\definecolor{crimson2143940}{RGB}{214,39,40}
\definecolor{darkgray176}{RGB}{176,176,176}
\definecolor{darkorange25512714}{RGB}{255,127,14}
\definecolor{forestgreen4416044}{RGB}{44,160,44}
\definecolor{lightgray204}{RGB}{204,204,204}
\definecolor{mediumpurple148103189}{RGB}{148,103,189}
\definecolor{sienna1408675}{RGB}{140,86,75}
\definecolor{steelblue31119180}{RGB}{31,119,180}

\begin{axis}[
legend cell align={left},
legend style={
  fill opacity=0.8,
  draw opacity=1,
  text opacity=1,
  at={(0.03,0.97)},
  anchor=north west,
  draw=lightgray204
},
tick align=outside,
tick pos=left,
x grid style={darkgray176},
xlabel={\(\displaystyle \log(1/\delta)\)},
xmajorgrids,
xmin=2.45474335933348, xmax=14.3564994721848,
xtick style={color=black},
y grid style={darkgray176},
ylabel={Communication complexity},
ymajorgrids,
ymin=36.4946875, ymax=88.2245625,
ytick style={color=black}
]
\addplot [semithick, darkorange25512714, mark=triangle*, mark size=3, mark options={solid}]
table {%
13.8155105579643 75.22775
12.2060726455302 74.275
11.5129254649702 74.40025
9.90348755253613 71.40525
9.21034037197618 69.36725
7.60090245954208 66.98225
6.90775527898214 67.06375
5.29831736654804 63.058
4.60517018598809 62.1355
2.99573227355399 54.6225
};
\addlegendentry{\#bits for QuBan ($\epsilon=2$)}
\addplot [semithick, forestgreen4416044, mark=triangle*, mark size=3, mark options={solid}]
table {%
13.8155105579643 76.32775
12.2060726455302 72.6165
11.5129254649702 71.81
9.90348755253613 68.33175
9.21034037197618 67.982
7.60090245954208 64.89
6.90775527898214 63.53625
5.29831736654804 60.76525
4.60517018598809 60.091
2.99573227355399 57.082
};
\addlegendentry{\#bits for QuBan ($\epsilon=0.5$)}
\addplot [semithick, crimson2143940, mark=diamond*, mark size=3, mark options={solid}]
table {%
13.8155105579643 54.68425
12.2060726455302 53.444
11.5129254649702 53.129
9.90348755253613 50.67775
9.21034037197618 49.954
7.60090245954208 49.23425
6.90775527898214 47.5585
5.29831736654804 46.3055
4.60517018598809 46.08575
2.99573227355399 40.581
};
\addlegendentry{\#bits for ICQ-SE ($B=2$)}
\addplot [semithick, sienna1408675, mark=diamond*, mark size=3, mark options={solid}]
table {%
13.8155105579643 59.44425
12.2060726455302 57.6635
11.5129254649702 57.6195
9.90348755253613 52.60125
9.21034037197618 52.02
7.60090245954208 51.422
6.90775527898214 50.1515
5.29831736654804 48.3005
4.60517018598809 47.2865
2.99573227355399 43.722
};
\addlegendentry{\#bits for ICQ-SE ($B=3$)}
\end{axis}

\end{tikzpicture}} }

\subfloat[\label{fig: Hardness_round}
]{\scalebox{0.55}{\begin{tikzpicture}

\definecolor{crimson2143940}{RGB}{214,39,40}
\definecolor{darkgray176}{RGB}{176,176,176}
\definecolor{darkorange25512714}{RGB}{255,127,14}
\definecolor{forestgreen4416044}{RGB}{44,160,44}
\definecolor{lightgray204}{RGB}{204,204,204}
\definecolor{mediumpurple148103189}{RGB}{148,103,189}
\definecolor{sienna1408675}{RGB}{140,86,75}
\definecolor{steelblue31119180}{RGB}{31,119,180}
\definecolor{orchid227119194}{RGB}{227,119,194}
\begin{axis}[
legend cell align={left},
legend style={
  fill opacity=0.8,
  draw opacity=1,
  text opacity=1,
  at={(0.97,0.97)},
  anchor=north east,
  draw=lightgray204
},
tick align=outside,
tick pos=left,
x grid style={darkgray176},
xlabel={\(\displaystyle \Delta\)},
xmajorgrids,
xmin=0.05, xmax=1.05,
xtick style={color=black},
y grid style={darkgray176},
ylabel={Round complexity},
ymajorgrids,
ymin=1.5567375, ymax=12.0355125,
ytick style={color=black}
]
\addplot [semithick, darkorange25512714, mark=triangle*, mark size=3, mark options={solid}]
table {%
0.1 11.18775
0.147368421052632 10.06525
0.194736842105263 9.09675
0.242105263157895 8.486
0.289473684210526 7.91675
0.336842105263158 7.331
0.384210526315789 6.98075
0.431578947368421 6.74875
0.478947368421053 6.3075
0.526315789473684 6.017
0.573684210526316 5.88275
0.621052631578947 5.683
0.668421052631579 5.34875
0.715789473684211 5.07075
0.763157894736842 4.83325
0.810526315789474 4.66975
0.857894736842105 4.48625
0.905263157894737 4.3195
0.952631578947368 4.139
1 4.019
};
\addlegendentry{QuBan ($\epsilon=2$)}
\addplot [semithick, forestgreen4416044, mark=triangle*, mark size=3, mark options={solid}]
table {%
0.1 9.83725
0.147368421052632 8.50425
0.194736842105263 7.717
0.242105263157895 6.9705
0.289473684210526 6.33825
0.336842105263158 5.94775
0.384210526315789 5.49475
0.431578947368421 5.05975
0.478947368421053 4.87625
0.526315789473684 4.5235
0.573684210526316 4.159
0.621052631578947 3.97675
0.668421052631579 3.8175
0.715789473684211 3.543
0.763157894736842 3.26225
0.810526315789474 3.072
0.857894736842105 2.95875
0.905263157894737 2.85025
0.952631578947368 2.6795
1 2.47175
};
\addlegendentry{QuBan ($\epsilon=0.5$)}
\addplot [semithick, crimson2143940, mark=diamond*, mark size=3, mark options={solid}]
table {%
0.1 10.96225
0.147368421052632 9.57725
0.194736842105263 8.99425
0.242105263157895 8.33175
0.289473684210526 7.88025
0.336842105263158 7.0795
0.384210526315789 6.8955
0.431578947368421 6.20725
0.478947368421053 6.1725
0.526315789473684 5.8825
0.573684210526316 5.725
0.621052631578947 5.2755
0.668421052631579 5.009
0.715789473684211 4.9785
0.763157894736842 4.81225
0.810526315789474 4.3615
0.857894736842105 4.147
0.905263157894737 4.178
0.952631578947368 4.0935
1 3.84325
};
\addlegendentry{ICQ-SE ($B=2$)}
\addplot [semithick, sienna1408675, mark=diamond*, mark size=3, mark options={solid}]
table {%
0.1 10.03325
0.147368421052632 8.872
0.194736842105263 7.95675
0.242105263157895 7.151
0.289473684210526 6.64425
0.336842105263158 6.04775
0.384210526315789 5.90625
0.431578947368421 5.75275
0.478947368421053 5.37125
0.526315789473684 4.802
0.573684210526316 4.5025
0.621052631578947 4.216
0.668421052631579 4.0015
0.715789473684211 3.929
0.763157894736842 3.8945
0.810526315789474 3.8585
0.857894736842105 3.6265
0.905263157894737 3.2715
0.952631578947368 3.046
1 2.99075
};
\addlegendentry{ICQ-SE ($B=3$)}
\addplot [semithick, orchid227119194, mark=square*, mark size=3, mark options={solid}]
table {%
0.1 9.009
0.147368421052632 7.90725
0.194736842105263 6.9705
0.242105263157895 6.21175
0.289473684210526 5.80425
0.336842105263158 5.18225
0.384210526315789 4.8715
0.431578947368421 4.4725
0.478947368421053 4.0775
0.526315789473684 3.867
0.573684210526316 3.59125
0.621052631578947 3.24275
0.668421052631579 3.02175
0.715789473684211 2.87125
0.763157894736842 2.66525
0.810526315789474 2.44025
0.857894736842105 2.198
0.905263157894737 2.033
0.952631578947368 1.92075
1 1.80475
};
\addlegendentry{Unquantized SE}
\end{axis}

\end{tikzpicture}} }
\subfloat[\label{fig: Hardness_sample}
]{\scalebox{0.55}{
\begin{tikzpicture}

\definecolor{crimson2143940}{RGB}{214,39,40}
\definecolor{darkgray176}{RGB}{176,176,176}
\definecolor{darkorange25512714}{RGB}{255,127,14}
\definecolor{forestgreen4416044}{RGB}{44,160,44}
\definecolor{lightgray204}{RGB}{204,204,204}
\definecolor{mediumpurple148103189}{RGB}{148,103,189}
\definecolor{sienna1408675}{RGB}{140,86,75}
\definecolor{steelblue31119180}{RGB}{31,119,180}
\definecolor{orchid227119194}{RGB}{227,119,194}
\begin{axis}[
legend cell align={left},
legend style={
  fill opacity=0.8,
  draw opacity=1,
  text opacity=1,
  at={(0.97,0.97)},
  anchor=north east,
  draw=lightgray204
},
tick align=outside,
tick pos=left,
x grid style={darkgray176},
xlabel={\(\displaystyle \Delta\)},
xmajorgrids,
xmin=0.055, xmax=1.045,
xtick style={color=black},
y grid style={darkgray176},
ylabel={Sample complexity},
ymajorgrids,
ymin=-536.60525, ymax=11612.85625,
ytick style={color=black}
]
\addplot [semithick, darkorange25512714, mark=triangle*, mark size=3, mark options={solid}]
table {%
0.1 11060.608
0.147368421052632 4913.024
0.194736842105263 2560.64
0.242105263157895 1637.76
0.289473684210526 1065.936
0.336842105263158 754.408
0.384210526315789 579.056
0.431578947368421 473.672
0.478947368421053 370.652
0.526315789473684 301.656
0.573684210526316 260.66
0.621052631578947 226.313
0.668421052631579 185.306
0.715789473684211 157.066
0.763157894736842 127.669
0.810526315789474 112.088
0.857894736842105 100.229
0.905263157894737 92.113
0.952631578947368 83.049
1 75.969
};
\addlegendentry{QuBan ($\epsilon=2$)}
\addplot [semithick, forestgreen4416044, mark=triangle*, mark size=3, mark options={solid}]
table {%
0.1 4113.472
0.147368421052632 1657.824
0.194736842105263 940.048
0.242105263157895 583.128
0.289473684210526 376.012
0.336842105263158 285.412
0.384210526315789 205.744
0.431578947368421 157.624
0.478947368421053 133.486
0.526315789473684 104.007
0.573684210526316 83.829
0.621052631578947 72.886
0.668421052631579 63.403
0.715789473684211 52.793
0.763157894736842 44.4865
0.810526315789474 39.4185
0.857894736842105 35.783
0.905263157894737 32.3455
0.952631578947368 28.7635
1 25.2835
};
\addlegendentry{QuBan ($\epsilon=0.5$)}
\addplot [semithick, crimson2143940, mark=diamond*, mark size=3, mark options={solid}]
table {%
0.1 8820.096
0.147368421052632 3513.056
0.194736842105263 2212.432
0.242105263157895 1495.328
0.289473684210526 1085.84
0.336842105263158 614.4
0.384210526315789 566.692
0.431578947368421 350.752
0.478947368421053 331.212
0.526315789473684 263.998
0.573684210526316 238.217
0.621052631578947 181.273
0.668421052631579 146.831
0.715789473684211 146.11
0.763157894736842 133.513
0.810526315789474 98.291
0.857894736842105 77.916
0.905263157894737 75.095
0.952631578947368 69.496
1 59.252
};
\addlegendentry{ICQ-SE ($B=2$)}
\addplot [semithick, sienna1408675, mark=diamond*, mark size=3, mark options={solid}]
table {%
0.1 4760.64
0.147368421052632 2123.232
0.194736842105263 1123.92
0.242105263157895 678.608
0.289473684210526 469.512
0.336842105263158 310.932
0.384210526315789 252.772
0.431578947368421 224.598
0.478947368421053 183.142
0.526315789473684 130.646
0.573684210526316 103.355
0.621052631578947 86.049
0.668421052631579 74.602
0.715789473684211 67.492
0.763157894736842 62.962
0.810526315789474 60.771
0.857894736842105 54.743
0.905263157894737 45.63
0.952631578947368 39.054
1 36.843
};
\addlegendentry{ICQ-SE ($B=3$)}
\addplot [semithick, orchid227119194, mark=square*, mark size=3, mark options={solid}]
table {%
0.1 2438.432
0.147368421052632 1091.616
0.194736842105263 581.984
0.242105263157895 345.524
0.289473684210526 249.616
0.336842105263158 168.794
0.384210526315789 132.404
0.431578947368421 100.106
0.478947368421053 78.305
0.526315789473684 65.7985
0.573684210526316 53.802
0.621052631578947 43.5865
0.668421052631579 37.2605
0.715789473684211 32.878
0.763157894736842 28.363
0.810526315789474 24.481
0.857894736842105 21.0305
0.905263157894737 18.8195
0.952631578947368 17.1225
1 15.643
};
\addlegendentry{Unquantized SE}
\end{axis}

\end{tikzpicture}} }
\subfloat[\label{fig: Hardness_bits}
]{\scalebox{0.55}{
\begin{tikzpicture}

\definecolor{crimson2143940}{RGB}{214,39,40}
\definecolor{darkgray176}{RGB}{176,176,176}
\definecolor{darkorange25512714}{RGB}{255,127,14}
\definecolor{forestgreen4416044}{RGB}{44,160,44}
\definecolor{lightgray204}{RGB}{204,204,204}
\definecolor{mediumpurple148103189}{RGB}{148,103,189}
\definecolor{sienna1408675}{RGB}{140,86,75}
\definecolor{steelblue31119180}{RGB}{31,119,180}

\begin{axis}[
legend cell align={left},
legend style={
  fill opacity=0.8,
  draw opacity=1,
  text opacity=1,
  at={(0.97,0.97)},
  anchor=north east,
  draw=lightgray204
},
tick align=outside,
tick pos=left,
x grid style={darkgray176},
xlabel={\(\displaystyle \Delta\)},
xmajorgrids,
xmin=0.055, xmax=1.045,
xtick style={color=black},
y grid style={darkgray176},
ylabel={Communication complexity},
ymajorgrids,
ymin=36.8123875, ymax=181.3163625,
ytick style={color=black}
]
\addplot [semithick, darkorange25512714, mark=triangle*, mark size=3, mark options={solid}]
table {%
0.1 171.537
0.147368421052632 152.8695
0.194736842105263 138.43375
0.242105263157895 128.1985
0.289473684210526 118.213
0.336842105263158 111.229
0.384210526315789 105.46025
0.431578947368421 100.609
0.478947368421053 96.071
0.526315789473684 91.6845
0.573684210526316 87.88775
0.621052631578947 84.81175
0.668421052631579 81.0675
0.715789473684211 77.79875
0.763157894736842 72.723
0.810526315789474 69.768
0.857894736842105 67.59925
0.905263157894737 66.40175
0.952631578947368 65.05275
1 63.5145
};
\addlegendentry{\#bits for QuBan ($\epsilon=2$)}
\addplot [semithick, forestgreen4416044, mark=triangle*, mark size=3, mark options={solid}]
table {%
0.1 174.748
0.147368421052632 152.0115
0.194736842105263 137.542
0.242105263157895 126.249
0.289473684210526 115.61625
0.336842105263158 109.1935
0.384210526315789 101.2745
0.431578947368421 95.387
0.478947368421053 90.742
0.526315789473684 84.85425
0.573684210526316 79.96325
0.621052631578947 76.38425
0.668421052631579 72.21825
0.715789473684211 67.578
0.763157894736842 63.7065
0.810526315789474 60.902
0.857894736842105 58.238
0.905263157894737 55.669
0.952631578947368 52.9635
1 49.937
};
\addlegendentry{\#bits for QuBan ($\epsilon=0.5$)}
\addplot [semithick, crimson2143940, mark=diamond*, mark size=3, mark options={solid}]
table {%
0.1 111.9525
0.147368421052632 98.6265
0.194736842105263 92.1725
0.242105263157895 85.923
0.289473684210526 81.62075
0.336842105263158 73.7805
0.384210526315789 72.4985
0.431578947368421 64.914
0.478947368421053 63.89375
0.526315789473684 61.7135
0.573684210526316 60.14825
0.621052631578947 56.2895
0.668421052631579 54.16475
0.715789473684211 54.39475
0.763157894736842 52.95975
0.810526315789474 48.55425
0.857894736842105 45.92975
0.905263157894737 45.37875
0.952631578947368 44.62325
1 43.38075
};
\addlegendentry{\#bits for ICQ-SE ($B=2$)}
\addplot [semithick, sienna1408675, mark=diamond*, mark size=3, mark options={solid}]
table {%
0.1 147.47325
0.147368421052632 129.7515
0.194736842105263 116.1495
0.242105263157895 105.42225
0.289473684210526 96.9465
0.336842105263158 88.617
0.384210526315789 83.6845
0.431578947368421 80.93175
0.478947368421053 76.204
0.526315789473684 69.282
0.573684210526316 64.45725
0.621052631578947 61.1805
0.668421052631579 58.82025
0.715789473684211 56.72825
0.763157894736842 55.39025
0.810526315789474 54.85875
0.857894736842105 52.8935
0.905263157894737 49.7935
0.952631578947368 47.506
1 46.8295
};
\addlegendentry{\#bits for ICQ-SE ($B=3$)}
\end{axis}

\end{tikzpicture}} }
	\caption{\ifreport Figures \ref{fig: QSE_alpha_round}, \ref{fig: QSE_alpha_sample} and \ref{fig: QSE_alpha_bits} demonstrate the dependence of the performance of ICQ-SE on $\alpha$. Figures \ref{fig: QSE_B_round} and \ref{fig: QSE_B_sample} demonstrate the dependence on $\beta$. \fi
    Figures \ref{fig: bounded_round}, \ref{fig: bounded_sample} and \ref{fig: bounded_bits} compare ICQ-SE with QuBan \cite{hanna21} and Fed-SEL \cite{mitra_heterogeneity} for bounded rewards while Figures \ref{fig: Gaussian_round}, \ref{fig: Gaussian_sample} and \ref{fig: Gaussian_bits} compare ICQ-SE with QuBan for unbounded rewards. Finally, Figures \ref{fig: Hardness_round}, \ref{fig: Hardness_sample} and \ref{fig: Hardness_bits} compare the dependence of ICQ-SE and QuBan on the hardness of the underlying instance.  
 } \label{fig: sims}
\end{figure*}

\ifreport

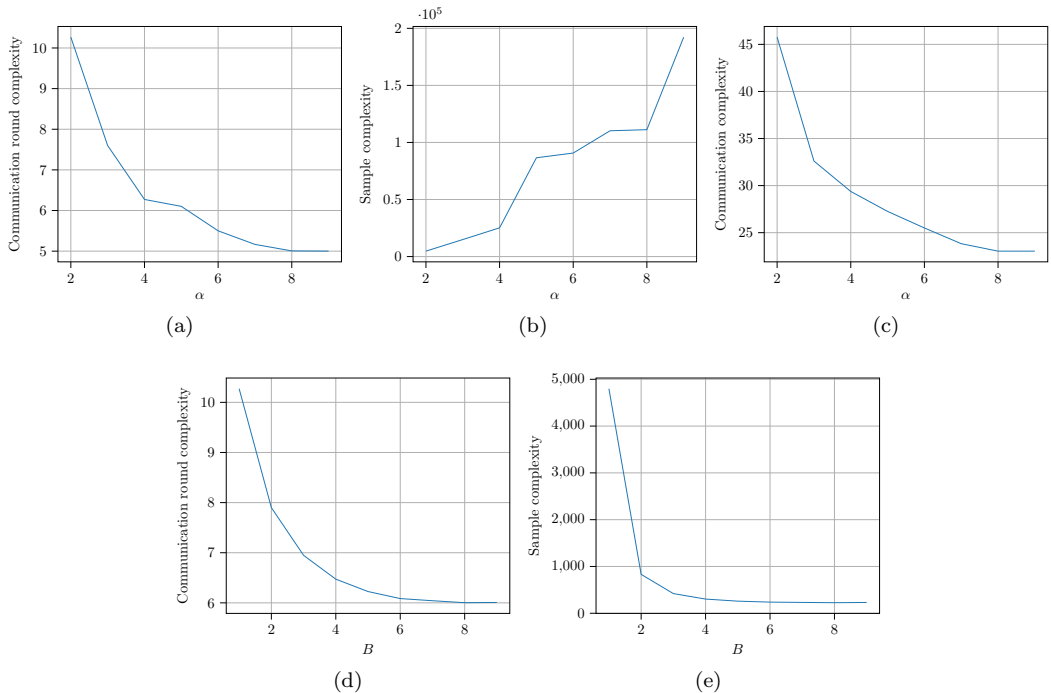
\begin{figure*}[!htbp]
	\centering
\subfloat[\label{fig: QSE_alpha_round}
]{\scalebox{0.55}{
\begin{tikzpicture}

\definecolor{darkgray176}{RGB}{176,176,176}
\definecolor{darkorange25512714}{RGB}{255,127,14}
\definecolor{lightgray204}{RGB}{204,204,204}
\definecolor{steelblue31119180}{RGB}{31,119,180}

\begin{axis}[
tick align=outside,
tick pos=left,
x grid style={darkgray176},
xlabel={\(\displaystyle \alpha\)},
xmajorgrids,
xmin=1.65, xmax=9.35,
xtick style={color=black},
y grid style={darkgray176},
ylabel={Communication round complexity},
ymajorgrids,
ymin=4.7367, ymax=10.5293,
ytick style={color=black}
]
\addplot [semithick, steelblue31119180]
table {%
2 10.266
3 7.598
4 6.272
5 6.102
6 5.5
7 5.166
8 5.006
9 5
};
\end{axis}

\end{tikzpicture}} }
\subfloat[\label{fig: QSE_alpha_sample}
]{\scalebox{0.55}{
\begin{tikzpicture}
	
	\definecolor{darkgray176}{RGB}{176,176,176}
	\definecolor{darkorange25512714}{RGB}{255,127,14}
	\definecolor{lightgray204}{RGB}{204,204,204}
	\definecolor{steelblue31119180}{RGB}{31,119,180}
	
	\begin{axis}[
		tick align=outside,
		tick pos=left,
		x grid style={darkgray176},
		xlabel={\(\displaystyle \alpha\)},
		xmajorgrids,
		xmin=1.65, xmax=9.35,
		xtick style={color=black},
		y grid style={darkgray176},
		ylabel={Sample complexity},
		ymajorgrids,
		ymin=-4624.6524, ymax=201639.1724,
		ytick style={color=black}
		]
		\addplot [semithick, steelblue31119180]
		table {%
			2 4750.976
			3 14902.056
			4 25167.872
			5 86570
			6 90720
			7 110225.108
			8 111198.208
			9 192263.544
		};
	\end{axis}
	
\end{tikzpicture}} }
\subfloat[\label{fig: QSE_alpha_bits}
]{\scalebox{0.55}{
\begin{tikzpicture}
	
	\definecolor{darkgray176}{RGB}{176,176,176}
	\definecolor{darkorange25512714}{RGB}{255,127,14}
	\definecolor{lightgray204}{RGB}{204,204,204}
	\definecolor{steelblue31119180}{RGB}{31,119,180}
	
	\begin{axis}[
		tick align=outside,
		tick pos=left,
		x grid style={darkgray176},
		xlabel={\(\displaystyle \alpha\)},
		xmajorgrids,
		xmin=1.65, xmax=9.35,
		xtick style={color=black},
		y grid style={darkgray176},
		ylabel={Communication complexity},
		ymajorgrids,
		ymin=21.9016, ymax=46.9024,
		ytick style={color=black}
		]
		\addplot [semithick, steelblue31119180]
		table {%
			2 45.766
			3 32.608
			4 29.378
			5 27.258
			6 25.5
			7 23.83
			8 23.038
			9 23.038
		};
	\end{axis}
	
\end{tikzpicture}} }

\subfloat[\label{fig: QSE_B_round}
]{\scalebox{0.55}{
\begin{tikzpicture}

\definecolor{darkgray176}{RGB}{176,176,176}
\definecolor{steelblue31119180}{RGB}{31,119,180}

\begin{axis}[
tick align=outside,
tick pos=left,
x grid style={darkgray176},
xlabel={\(\displaystyle B\)},
xmajorgrids,
xmin=0.6, xmax=9.4,
xtick style={color=black},
y grid style={darkgray176},
ylabel={Communication round complexity},
ymajorgrids,
ymin=5.78976, ymax=10.48544,
ytick style={color=black}
]
\addplot [semithick, steelblue31119180]
table {%
1 10.272
2 7.8996
3 6.9478
4 6.4728
5 6.2258
6 6.0856
7 6.0426
8 6.0032
9 6.008
};
\end{axis}

\end{tikzpicture}} }
\subfloat[\label{fig: QSE_B_sample}
]{\scalebox{0.55}{
\begin{tikzpicture}
	
	\definecolor{darkgray176}{RGB}{176,176,176}
	\definecolor{steelblue31119180}{RGB}{31,119,180}
	
	\begin{axis}[
		tick align=outside,
		tick pos=left,
		x grid style={darkgray176},
		xlabel={\(\displaystyle B\)},
		xmajorgrids,
		xmin=0.6, xmax=9.4,
		xtick style={color=black},
		y grid style={darkgray176},
		ylabel={Sample complexity},
		ymajorgrids,
		ymin=-2.75134000000003, ymax=5027.00614,
		ytick style={color=black}
		]
		\addplot [semithick, steelblue31119180]
		table {%
			1 4798.3808
			2 831.3264
			3 420.2608
			4 303.344
			5 258.8036
			6 238.3352
			7 231.878
			8 225.874
			9 231.1376
		};
	\end{axis}
	
\end{tikzpicture}} }

 \caption{Figures \ref{fig: QSE_alpha_round}, \ref{fig: QSE_alpha_sample} and \ref{fig: QSE_alpha_bits} demonstrate the dependence of the performance of ICQ-SE on $\alpha$. Figures \ref{fig: QSE_B_round} and \ref{fig: QSE_B_sample} demonstrate the dependence on $\beta$.
 } \label{fig: sims_ext}
\end{figure*}
\fi
	\section{Conclusion}
	We propose ICQ, a novel quantization scheme for the distributed best-arm identification problem where the learner does not have access to full-precision rewards, and analyze ICQ-SE, which is the application of ICQ to the Successive Elimination algorithm for this setting. Future lines of work include: (1) using a variable-length and adaptive quantization scheme in each round to reduce the communication complexity; for example, a Lloyd-Max quantizer based on the empirical distribution, (2) characterizing a lower bound on the communication complexity required to ensure a certain sample/round complexity, and (3) developing quantization schemes for the fixed budget variant of the best-arm identification problem.
	
	\bibliographystyle{IEEEtran}  
	\bibliography{main.bib}
	
\ifreport
\appendix \section{Proofs} \ifreport
\begin{proof}[Proof of Lemma \ref{lem: equivalence}]
	We wish to show that for all arms $1 \leq j \leq K$ and any $\delta > 0$,
	\begin{align*}
		\bigcup_{i = 1}^\infty \left\{ |\mut_{j,i} - \mu_j | > U(i,\delta) \right\} \subseteq \bigcup_{i=1}^\infty \left\{ |\muh_{j,i} - \mu_j | > U^\prime(i,\delta) \right\}.
	\end{align*}
	First observe that the sets on the left-hand side are empty when $i = 0$, as 
	\begin{align}
		\left\{ |\mut_{j,0} - \mu_j | >  U(0, \delta) \right\} =  \left\{ |\mut_{j,0} - \mu_j | >  (b-a) \right\} = \emptyset, \label{eqn: U_0_defn}
	\end{align}
	so it is sufficient to show that
	\begin{align*}
		\bigcup_{i = 0}^\infty \left\{ |\mut_{j,i} - \mu_j | > U(i,\delta) \right\} \subseteq \bigcup_{i=1}^\infty \left\{ |\muh_{j,i} - \mu_j | > U^\prime(i,\delta) \right\}.
	\end{align*}
	We do so by instead proving the following equivalent statement --- for each $0 \leq i < \infty$,
	\begin{equation*}
		\left\{ |\mut_{j,i} - \mu_j | > U(i,\delta) \right\} \subseteq \bigcup_{k=1}^\infty \left\{ |\muh_{j,k} - \mu_j | > U^\prime(k,\delta) \right\}.		
	\end{equation*}
	The proof is by induction over $i$. The base case, $i=0$, holds trivially as the left-hand side is $\emptyset$.
	Assume now that this property holds for index $i-1 \geq 0$, i.e., 
	\begin{equation*}
		\left\{ |\mut_{j,i-1} - \mu_j | > U(i-1,\delta) \right\} \subseteq \bigcup_{k=1}^\infty \left\{ |\muh_{j,k} - \mu_j | > U^\prime(k,\delta) \right\}.
	\end{equation*}
	Using \eqref{eqn: triangle} and \eqref{eqn: U_breakup} respectively, we have
	\begin{align*}
		&\left\{ |\mut_{j,i} - \mu_j | > U(i,\delta) \right\} \\
		\subseteq &\left\{ |\mut_{j,i} - \muh_{j,i} | + |\muh_{j,i} - \mu_j | > U(i,\delta) \right\},\\
		\subseteq &\left\{ |\muh_{j,i} - \mu_j | > U^\prime(i,\delta) \right\} \cup\\& \left\{ |\mut_{j,i} - \muh_{j,i} | > \frac{1}{2^{B}}U^\prime(i,\delta) + \frac{1}{2^{B}} U(i-1, \delta) \right\},
	\end{align*}
	where the second step follows as $U(i,\delta) = U^\prime(i, \delta) + \frac{1}{2^{B}}U^\prime(i,\delta) + \frac{1}{2^{B}} U(i-1, \delta)$
	and $a + b > x + y$ requires at least one of $a > x$ or $b > y$.
	
	We now bound the second term in the union on the right-hand side. The estimated mean is encoded using $\enc$ on the interval $[\lcb(j,i-1,\delta)-U^\prime(i,\delta),\ucb(j,i-1,\delta)+U^\prime(i,\delta)]$. As long as $\muh_{j,i}$ lies in this interval, the learner can decode $\mut_{j,i}$ to an error within $\frac{1}{2} \cdot \frac{1}{2^{B}}$ times the width of this interval, given by $2 \left (U^\prime(i,\delta) + U(i-1, \delta) \right )$. Hence we have
	\begin{align*}
		&\left\{ |\mut_{j,i} - \muh_{j,i} | > \frac{1}{2^{B}}U^\prime(i,\delta) + \frac{1}{2^{B}} U(i-1, \delta) \right\} \\
		\subseteq &\left\{ |\mut_{j,i-1} - \muh_{j,i} | >  U^\prime(i,\delta) + U(i-1, \delta) \right\} \\
		\subseteq &\left\{ |\mut_{j,i-1} - \mu_{j} | +  |\muh_{j,i} - \mu_{j} | >  U^\prime(i,\delta) + U(i-1, \delta) \right\} \\
		\subseteq &\left\{ |\mut_{j,i-1} - \mu_j | >  U(i-1, \delta) \right\} \cup \left\{ |\muh_{j,i} - \mu_j | > U^\prime(i,\delta)  \right\} \\
		\subseteq &\bigcup_{k=1}^\infty \left\{ |\muh_{j,k} - \mu_j | > U^\prime(k,\delta) \right\}, \numberthis \label{eqn: error_event_final}
	\end{align*}
	where the last step follows from the induction hypothesis. 
\end{proof}
\fi

\begin{proof}[Proof of Lemma \ref{lem: delta}] This follows easily by a union bound, as
	\begin{align*}
		\prob(\mathcal{E})&= \prob \left( \bigcup_{j =1}^K\bigcup_{i=1}^\infty  \left\{ |\mut_{j,i} - \mu_j | > U(i,\delta) \right\}  \right) \\
		&\stackrel{(a)}\leq \prob \left( \bigcup_{j =1}^K\bigcup_{i=1}^\infty  \left\{ |\muh_{j,i} - \mu_j | > U^\prime(i,\delta) \right\}  \right) \\
         &\leq \prob \left( \bigcup_{j =1}^K\bigcup_{i=1}^\infty  \left\{ |\muh_{j,i} - \mu_j | > \sigma \sqrt{\frac{2 \log( 4Kt_i^2/\delta)}{t_i}} \right\}  \right) \\
         &\leq \sum_{j =1}^K\sum_{i=1}^\infty \prob \left( \left|\frac{1}{t_i}\sum_{k=1}^{t_i} r_{j, k} - \mu_j \right| > \sigma \sqrt{\frac{2 \log( 4Kt_i^2/\delta)}{t_i}} \right) \\ 
         &\leq \sum_{j =1}^K\sum_{i=1}^\infty \prob \left(\left|\frac{1}{t_i}\sum_{k=1}^{t_i} r_{j, k} - \mu_j \right| > \sigma \sqrt{\frac{2 \log( 4Kt_i^2/\delta)}{t_i}} \right) \\
       &\stackrel{(b)}\leq  \sum_{j =1}^K\sum_{i=1}^\infty 2\exp\left(-\log( 4Kt_i^2/\delta)\right) =  \sum_{j =1}^K\sum_{i=1}^\infty 2 \left(\frac{\delta}{4Kt_i^2}\right)
        \\& \stackrel{(c)}<  \sum_{j =1}^K \frac{\delta}{K} = \delta,
	\end{align*}
	where (a) follows from Lemma \ref{lem: equivalence}, (b) from the definition of $\sigma^2$-subgaussianity, and (c) from $\sum_{i=1}^\infty \frac{1}{t_i^2}\leq \sum_{i=1}^\infty \frac{1}{i^2} < 2$.
\end{proof}
\begin{proof}[Proof of Theorem \ref{thm: best_arm}] For arm $k$ to be dropped from the active set $S$ at round $i$, there must exist an arm $l$ such that
	\begin{align*}
		\lcb(l, i, \delta) &\geq \ucb(k, i, \delta)\\
		\iff \mut_{l,i} - U(i, \delta)  &\geq \mut_{k,i} + U(i,\delta).
	\end{align*}
	On the set $\mathcal{E}^c$ (with $\mathcal{E}$ as defined in Lemma \ref{lem: delta}), we have for all arms $j$ at round $i$, $	|\mut_{j,i} - \mu_j | \leq U(i,\delta)$. In particular, for arms $k$ and $l$, we have
	\begin{equation*}
		\mu_l + U(i, \delta) \geq \mut_{l,i}\text{ and } \mut_{k,i} \geq  \mu_k - U(i, \delta).
	\end{equation*}
	Combining the above inequalities, we have that arm $k$ can be dropped from $S$ only if there exists an arm $l$ such that
	\begin{align*}
		\mu_l + U(i, \delta) - U(i, \delta)  &\geq \mu_k - U(i, \delta) + U(i,\delta)\\
		\iff \mu_l &\geq \mu_k.
	\end{align*}
	We thus have that the best arm will always remain in the active set under $\mathcal{E}^c$, which completes the proof.
\end{proof}

\begin{proof}[Proof of Lemma \ref{lem: U_ub}]
	Let $i \geq 1$. From the definition of $U(i,\delta)$ and using the result that for $j \leq i$, $t_j \leq t_i$, it follows that
	\begin{align*}
		U(i,\delta) &=U^\prime(i, \delta) + \frac{1}{2^{B}}\left[U^\prime(i,\delta) +  U(i-1, \delta)\right]\\
		&= \left(1+\frac{1}{2^{B}}\right) \sum_{j=1}^{i} \left (\frac{1}{2^{B}} \right )^{i-j} U^\prime(j,\delta) {\ifreport\else\\& \hspace{2em}\fi}+ \left(\frac{1}{2^{B}}\right)^{i}(b-a) \\
		&=  \left(1+\frac{1}{2^{B}}\right) {\ifreport\else\times \\ &\hspace{2em}\fi} \sum_{j=1}^{i} \left (\frac{1}{2^{B}} \right )^{i-j} \sigma \sqrt{\frac{2 \log \left(4 K t_j^2 / \delta \right)}{t_j}} {\ifreport\else\\& \hspace{2em}\fi}+ \left(\frac{1}{2^{B}}\right)^{i}(b-a) \\
		&\leq \sigma \left(1+\frac{1}{2^{B}}\right)  \sqrt{2 \log \left(4 K t_i^2 / \delta \right)} {\ifreport\else\times \\& \hspace{2em}\fi} \sum_{j=1}^{i} \left (\frac{1}{2^{B}} \right )^{i-j} \frac{1}{\sqrt{t_j}} + \left(\frac{1}{2^{B}}\right)^{i}(b-a).
	\end{align*}
	Now, using $t_i = \alpha^i$ gives us that
	\begin{align*}
		U(i,\delta) &\leq \sigma \left(1+\frac{1}{2^{B}}\right)  \sqrt{2 \log \left(4 K t_i^2 / \delta \right)} {\ifreport\else\times \\& \hspace{2em}\fi} \left (\frac{1}{2^{B}} \right )^{i} \sum_{j=1}^{i} \left (\frac{2^{B} }{\sqrt{\alpha}} \right )^{j} {\ifreport\else\\& \hspace{2em}\fi} + \left(\frac{1}{2^{B}}\right)^{i}(b-a).
	\end{align*}
	As $\sqrt{\alpha} < 2^B$, it follows that
	\begin{align*}
		U(i,\delta) &\leq \sigma \left(1+\frac{1}{2^{B}}\right)  \sqrt{2 \log \left(4 K t_i^2 / \delta \right)} {\ifreport\else\times\\& \hspace{2em}\fi} \left (\frac{1}{2^{B}} \right )^{i} \frac{2^{B} }{ \sqrt{\alpha}} \frac{\left(\frac{2^{B}}{ \sqrt{\alpha}}\right)^i-1}{\frac{2^{B}}{ \sqrt{\alpha}}-1}{\ifreport\else\\& \hspace{2em}\fi} + \left(\frac{1}{2^{B}}\right)^{i}(b-a) \\
		&\leq \sigma \left(1+\frac{1}{2^{B}}\right)  \sqrt{2 \log \left(4 K t_i^2 / \delta \right)} {\ifreport\else\times\\& \hspace{2em}\fi} \left (\frac{1}{2^{B}} \right )^{i} \frac{2^{B} }{ \sqrt{\alpha}} \frac{\left(\frac{2^{B}}{ \sqrt{\alpha}}\right)^i}{\frac{2^{B}}{ \sqrt{\alpha}}-1}{\ifreport\else\\& \hspace{2em}\fi} \hspace{2em} + \left(\frac{1}{2^{B}}\right)^{i}(b-a) \\
		&= \sigma \left(1+\frac{1}{2^{B}}\right)  \sqrt{2 \log \left(4 K t_i^2 / \delta \right)} {\ifreport\else\times\\& \hspace{2em}\fi}  \frac{2^{B}}{\sqrt{t_i}} \frac{1}{{2^{B}} - { \sqrt{\alpha}}}{\ifreport\else\\& \hspace{2em}\fi} + \left(\frac{1}{2^{B}}\right)^{i}(b-a) \\
		&= \left(1+\frac{1}{2^{B}}\right) \frac{2^B }{{2^{B}} - { \sqrt{\alpha}}} U'(i,\delta) {\ifreport\else\\& \hspace{2em}\fi} + \left(\frac{1}{2^{B}}\right)^{i}(b-a). 
	\end{align*}
	To show that
	\begin{align*}
		U(i,\delta) \leq 2\left(1+\frac{1}{2^{B}}\right) \frac{2^B }{{2^{B}} - { \sqrt{\alpha}}} U'(i,\delta),
	\end{align*}
	it suffices to show that 
	\begin{align*}
		\left(\frac{1}{2^{B}}\right)^{i}(b-a) \leq	\left(1+\frac{1}{2^{B}}\right) \frac{2^B}{{2^{B}} - {\sqrt{\alpha}}} U'(i,\delta).
	\end{align*}
	As 
	\begin{align*}
		\left(1+\frac{1}{2^{B}}\right) \frac{2^B }{{2^{B}} - { \sqrt{\alpha}}} \geq 1,
	\end{align*}
	for $i \geq 1$, and the denominator of $U'(i,\delta)$ satisfies the following result:
	\begin{align*}
		\sqrt{t_i} = \left (\sqrt{\alpha} \right )^i < \left ({2^B} \right )^i,
	\end{align*}  
	it suffices to show that 
	$\sigma \sqrt{2 \log \left(4Kt_i^2 / \delta \right)} \geq (b-a)$. This will hold as long as
	\begin{align*}
		i\geq \log_{\alpha} \left (\sqrt{\frac{\delta}{4K}\exp \left (\frac{(b-a)^2}{2 \sigma^2} \right )} \right ).
	\end{align*}
	For a small enough $\delta$, this will be satisfied trivially. If
	\begin{align*}
		\delta < 4K \alpha^2 \exp \left(-(b-a)^2/(2\sigma^2)\right),
	\end{align*}
	then we have that for all $i \geq 1$,
	\begin{align*}
		U(i,\delta) \leq 2\left(1+\frac{1}{2^{B}}\right) \frac{2^B }{{2^{B}} - { \sqrt{\alpha}}} U'(i,\delta).
	\end{align*}
\end{proof}

\begin{proof}[Proof of Theorem \ref{thm: sample_complexity}]
	Assume w.l.o.g.\ that the optimal arm is arm 1. \ifreport Similar to the proof of Theorem \ref{thm: best_arm}, a \else A \fi suboptimal arm $j$ will be removed if
	\begin{align}
		\lcb(1,i,\delta) > \ucb(j,i,\delta). \label{eqn: event_arm_kicked}
	\end{align}
	On $\mathcal{E}^{c}$, we have that
	\begin{align*}
		\mut_{1,i} &\geq \mu_1 - U(i,\delta), \\
		\mut_{j,i} &\leq \mu_j + U(i,\delta).
	\end{align*}
	The equation \eqref{eqn: event_arm_kicked} is guaranteed to occur if
	\begin{align*}
		\mu_1 - 2U(i,\delta) \geq \mu_j + 2U(i,\delta).
	\end{align*}
	Define $\Delta_j = \mu_1 - \mu_j$ to be the suboptimality gap of arm $j$ for $j \in \{2,\ldots,K\}$. The sample complexity of arm $j$ is the number of samples needed to remove the suboptimal arm $j$ from $S$. Define $T_j$ to be the smallest value of $t_i$ that satisfies 
	\begin{align*}
		\Delta_j \geq 4U(i,\delta).
	\end{align*}
	Then $\alpha  T_j$ is an upper bound on the sample complexity of arm $j$. The factor of $\alpha$ is needed here as we communicate exponentially sparsely only.
	For $t_i = \alpha^i$, we have, from Lemma \ref{lem: U_ub},
	\begin{align*}
		U(i,\delta) \leq 2c\, U^\prime(i, \delta),
	\end{align*}
	where 
	\begin{align*}
		c = \left(1 + \frac{1}{2^{B}}\right) \frac{2^{B}}{2^{B} - \sqrt{\alpha}}.
	\end{align*}
	If $T_j'$ is the smallest value of $t_i$ satisfying
	\begin{align*}
		U'(i,\delta) \leq \frac{\Delta_j}{8c},
	\end{align*}
	then $\alpha T_j'$ is an upper bound on $ \alpha T_j$ and thus an upper bound on the sample complexity of suboptimal arm $j$. Define $a=\frac{\Delta_j^2}{256c^2\sigma^2}$ and $b=\ln(4K/\delta)/2$. Then the smallest value of $i$ that is a solution to
	\begin{align*}
		U'(i,\delta) = \frac{\Delta_j}{8c}
	\end{align*}
	satisfies
	\begin{align*}
		-a t_i e^{-at_i} = - ae^{-b}.
	\end{align*}
	Let $\delta$ be small enough such that
	\begin{align*}
		ae^{-b} = \frac{\Delta_j^2}{256c^2 \sigma^2} e^{-\ln(4K/\delta)/2} = \frac{\Delta_j^2}{256c^2 \sigma^2} \sqrt{\frac{\delta}{4K}} < \frac{1}{e}. 
	\end{align*}
	Then the solution to this equation would be $\lceil \frac{-1}{a} W_{-1}(-ae^{-b})\rceil$, where $W_{-1}(y), \frac{-1}{e} <   y<0$ is the smallest value of $x<0$ satisfying $x e^x = y$. There are in fact two solutions, $W_0(y)$ and $W_{-1}(y)$, out of which $W_{-1}(y)$ is smaller (more negative) and belongs to $\left (-\infty,-1 \right )$. Using Theorem 3.1 of \cite{alzah_salem}, we have that $W_{-1}(y) > \frac{e}{e-1} \ln(-y)$. It follows that
	\begin{align*}
		\alpha T_j &\leq \frac{-\alpha}{a} \frac{e}{e-1} \ln(a e^{-b}) + 1 \\
		&= \frac{\alpha e}{e-1} \frac{256c^2 \sigma^2}{\Delta_j^2} (b-\ln a) + 1 \\
		&\leq \frac{410 \alpha c^2 \sigma^2}{\Delta_j^2} \ln \left(\frac{256c^2 \sigma^2 \sqrt{4K\delta}}{\Delta_j^2}\right) + 1.
	\end{align*}

\end{proof}
\fi
	
\end{document}